\documentclass[12pt]{article}
\usepackage{amstext,amssymb,amsmath,epsf,graphics}

\newtheorem{defn}{Definition}
\newtheorem{theorem}{Theorem}
\newtheorem{lemma}{Lemma}
\newtheorem{corollary}{Corollary}
\newtheorem{proposition}{Proposition}
\newtheorem{conject}{Conjecture}
\newcommand{\f}{f}

\newcommand{\g}{f^{-1}}

\renewcommand{\geq}{\geqslant}
\renewcommand{\leq}{\leqslant}

\newcommand{\bm}[1]{\boldsymbol #1}

\newcommand{\btheta}{\boldsymbol \theta}

\renewcommand{\Box}{\hspace*{\fill}~$\square$ \\}

\newcommand{\balpha}{\boldsymbol{\alpha}}

\newcommand{\RealR}{\mbox{$\mathbb{R}$}}
\newcommand{\Real}{\mbox{$\mathbb{R}$}}

\begin{document}

\title{Analytic Network Learning}
\date{}
\author{Kar-Ann Toh \\
{\footnotesize School of Electrical and Electronic Engineering}\\*[-1mm] {\footnotesize
Yonsei University, Seoul, Korea\ 03722}  } \maketitle

\begin{abstract}
    Based on the property that solving the system of linear matrix equations via
    the column space and the row space projections boils down to an approximation in
    the least squares error sense, a formulation for learning the weight
    matrices of the multilayer network can be derived. By exploiting into the vast
    number of feasible solutions of these interdependent weight matrices,
    the learning can be performed analytically layer
    by layer without needing of gradient computation after an initialization. Possible
    initialization schemes include utilizing the data matrix as initial weights
    and random initialization.
    The study is followed by an investigation into the representation capability and
    the output variance of the learning scheme. An extensive experimentation
    on synthetic and real-world data sets validates its numerical feasibility.
\end{abstract}

\section{Introduction}

\subsection{Background}

Attributed to its high learning capacity with good prediction capability, the deep neural
network has found its advantage in wide areas of science and engineering applications.
Such an observation has sparked a surge of investigations into the architectural and
learning aspects of the deep network for targeted applications.

The main ground for realizing the high learning capacity and predictivity comes from
several major advancements in the field which include the processing platform, the
learning regimen, and the availability of big data size. In terms of the processing
platform, the advancement in Graphics Processing Units (GPUs) has facilitated parallel
processing of complex network learning within accessible time. Together with the
relatively low cost of the hardware, the large number of public high level open source
libraries has enabled a crowdsourcing mode of learning architectural exploration. Based
on such a learning platform, several learning regimens such as the Convolutional Neural
Network (CNN or LeNet-5) \cite{LeCun3}, the AlexNet \cite{AlexK1}, the GoogLeNet or
Inception \cite{Szegedy1}, the Visual Geometry Group Network (VGG Net)
\cite{Simonyan14verydeep}, the Residual Network (ResNet) \cite{HeKaiming1} and the
DenseNet \cite{HuangGao1} have stretched the network learning in terms of the network
depth and prediction capability way beyond the known boundary established by the
conventional statistical methods.

Without sufficiently convincing explanation in theory, the advancement of deep learning
has been grounded upon `big' data, powerful machinery and crowd efforts to achieve at
`breakthrough' results that were not possible before. Such a swarming phenomenon has
pushed forward the demand of hardware as well as middle ware, but at the expense of
masking the importance of fundamental results available in statistical decision theory.
The research scene has arrived at such a state of deeming results unacceptable without
working directly on or comparing them with `big' data which implicitly relies on powerful
machinery.

\subsection{Motivation and Contributions}

Despite the great success in applications, understanding of the underneath learning
mechanism towards the representation and generalization properties gained from the
network depth is being far fetched and becoming imperative. From the perspective of
nonlinearity incurred by the activation functions in each layer, analyzing the learning
properties of the network becomes extremely difficult.

In terms of the optimality of network learning, several investigations in the literature
can be found. For the two-layer linear network, back in the year 1988, Baldi and Hornik
\cite{Hornik4,Baldi1} showed that the network was convex in each of its two weight
matrices and every local minimum was a global minimum. In 2012, Baldi and Lu
\cite{Baldi2} extended the result of convexity to deep linear network while conjecting
that every local minimum was a global minimum. Recently, Kawaguchi \cite{Kawaguchi1}
showed that the loss surface of deep linear networks was non-convex and non-concave,
every local minimum was a global minimum, every critical point that was not a global
minimum was a saddle point, and the rank of weight matrix at saddle points. These
observations, represented in the form of Directed Acyclic Graph (DAG), were carried
forward to the nonlinear networks with fewer unrealistic assumptions than that in
\cite{LeCun4,LeCun5} which were inspired from the Hamiltonian of the spherical spin-glass
model. Subsequently, Lu and Kawaguchi \cite{Kawaguchi2} showed that depth in linear
networks created no bad local minima. More recently, Yun et  al. \cite{YunCH1} provided
sufficient conditions for global optimality in deep nonlinear networks when the
activation functions satisfied certain invertibility, differentiability and bounding
conditions.

From the geometrical view of the loss function, Dinh et  al. \cite{YBengio2} argued that
an appropriate definition and use of minima in terms of their sharpness could help in
understanding the generalization behavior. With the understanding that the difficulty of
deep search was originated from the proliferation of saddle points and not local minima,
Dauphin et al. \cite{YBengio3} attempted a saddle-free Newton's method (see e.g.,
\cite{ChenKe2}) to search for the minima in deep networks. In \cite{LiHao1}, a
visualization technique based on filter normalization was introduced for studying the
loss landscape. By comparing networks with and without skip connections, the study showed
that the smoothness of loss surface depended highly on the network architecture. The
landscape of the empirical risk for deep learning of convolutional neural networks was
investigated in \cite{Poggio7}. By extending the case of linear networks in
\cite{Gunasekar1} to the case of nonlinear networks, Poggio et al. \cite{Poggio10} showed
that deep networks did not usually overfit the classification error for low-noise data
sets and the solution corresponded to margin maximization.

From the view of having a theoretical bound for the estimation error, in 2002, Langford
and Caruana \cite{Langford1} investigated into generalization in terms of the true error
rate bound of a distribution over a set of neural networks using the PAC-Bayes bound.
This approach had sparked off several follow-ups (see e.g.,
\cite{Dziugaite1,Neyshabur1,Neyshabur2}). Recently, Bartlett et al. \cite{Bartlett1}
proposed a margin-based generalization bound which was spectrally-normalized. Apart from
an empirical investigation of several capacity bounds based on the $\ell_2$,
$\ell_1$-path, $\ell_2$-path and the spectral norms \cite{Neyshabur3}, in
\cite{Neyshabur2}, the margin-based perturbation bound was combined with the PAC-Bayes
analysis to derive the generalization bound. Using the gap between the expected risk and
the empirical risk, Kawaguchi et al. \cite{Kawaguchi3} provided generalization bounds for
deep models which do not have explicit dependency on the number of weights, the depth and
the input dimensionality. In view of the lack of an analytic structure in learning, many
of these analyses were hinged upon the Stochastic Gradient Descent (SGD, see
e.g.,\cite{Brunelli11,Patrick11,ChenKe1}) search towards certain stationary points (e.g.,
\cite{Bartlett1,Poggio7,Dziugaite1,Neyshabur1}). Moreover, many of the analyses of the
linear network or the nonlinear one treated the network as a Directed Acyclic Graph (DAG)
model (see e.g., \cite{Gunasekar1,Neyshabur1,Kawaguchi3}) where it turned into the
complicated problem of topological sorting.

In \cite{Toh97,Toh98}, we have established a gradient-free learning approach that results
in an analytic learning of layered networks of fully connected structure. This simplifies
largely the analytical aspects of network learning. In this work, we further explore the
analytic learning on networks with receptive field. Particularly, the contributions of
current study include (1) establishment of an analytic learning framework where its goal
is to find a mapping matrix such that the target matrix falls within its range space. In
other words, a representative mapping is sought after based on the data matrix and its
Moore-Penrose inverse; (2) proposal of an analytic learning algorithm which utilizes the
sparse structural information and showcase two initialization possibilities; (3)
establishment of conditions for network representation on a finite set of data; (4)
investigation of the network output variance with respect to the network depth; and (5)
an extensive numerical study to validate the proposal.

\subsection{Organization}

The article is organized as follows. An introduction to several related concepts is given
in Section~\ref{sec_prelim} to pave the way for our development. Particularly, the
Moore-Penrose inverse and its relation to least squares error approximation is defined
and stated, together with the network structure of interest.
Section~\ref{sec_net_learning} presents the main results of network learning in analytic
form. Here, apart from a random initialization scheme, a deterministic initialization
scheme based on the data matrix is introduced. This sheds some lights on the underlying
learning mechanism. In Section~\ref{sec_net_representation}, the issue of network
representation is shown in view of a finite set of training data. The network
generalization is further explored via the output variance of the linear structure. In
Section~\ref{sec_synthetic}, two cases of synthetic data are studied to observe the
fitting behavior and the representation property of the proposed learning. In
Section~\ref{sec_expts}, the proposed learning is experimented on 42 data sets taken from
the public domain to validate its numerical feasibility. Finally, some concluding remarks
are given in Section~\ref{sec_conclusion}.

\section{Preliminaries} \label{sec_prelim}

\subsection{Moore-Penrose Inverse and Least Squares Solution}

\begin{defn} (see e.g., \cite{SLCampbell1,Albert1,Adi1,MacAusland1})
    A least squares solution to the system ${\bf A}{\btheta}={\bf y}$, where ${\bf
        A}\in\Real^{m\times d}$, ${\btheta}\in\Real^{d\times 1}$ and ${\bf y}\in\Real^{m\times
        1}$, is a vector $\hat{\btheta}$ such that
    \begin{equation}\label{eqn_LS_defn}
    \|\hat{\bf e}\| = \|{\bf A}\hat{\btheta}-{\bf y}\|
    \leq \|{\bf A}{\btheta}-{\bf y}\|
    =\|{\bf e}\|,
    \end{equation}
    where $\|\cdot\|$ denotes the $\ell^2$-norm of ``$\cdot$''.
\end{defn}

\begin{defn} (see e.g., \cite{Adi1,SLCampbell1,Albert1,MacAusland1}) \label{defn_pseudoinverse}
    If ${\bf A}$ (square or rectangular) is a matrix of real or complex elements, then there
    exits a unique matrix ${\bf A}^{\dag}$, known as the Moore-Penrose inverse or the
    pseudoinverse of ${\bf A}$, such that (i) ${\bf A}{\bf A}^{\dag}{\bf A}={\bf A}$, (ii)
    ${\bf A}^{\dag}{\bf A}{\bf A}^{\dag}={\bf A}^{\dag}$, (iii) $({\bf A}{\bf
        A}^{\dag})^*={\bf A}{\bf A}^{\dag}$ and (iv) $({\bf A}^{\dag}{\bf A})^*={\bf
        A}^{\dag}{\bf A}$ where $^*$ denotes the conjugate transpose.
\end{defn}

\begin{lemma} \label{thm_LSapprox}
    (see e.g., \cite{Albert1,SLCampbell1,Adi1,MacAusland1}) $\hat{\btheta}={\bf A}^{\dag}{\bf y}$ is the best
    approximate solution of ${\bf A}{\btheta}={\bf y}$.
\end{lemma}

If the linear system of equations ${\bf A}{\btheta}={\bf y}$ is an over-determined one
(i.e., $m>d$) and the columns of ${\bf A}$ are linearly independent, then ${\bf
    A}^{\dag}=({\bf A}^T{\bf A})^{-1}{\bf A}^T$ and ${\bf A}^{\dag}{\bf A}={\bf I}$. The
solution to the system can be derived via pre-multiplying ${\bf A}^{\dag}$ to both sides
of ${\bf A}{\btheta}={\bf y}$ giving ${\bf A}^{\dag}{\bf A}\hat{\btheta}={\bf
    A}^{\dag}{\bf y}$ which leads to the result in Lemma~\ref{thm_LSapprox} since ${\bf
    A}^{\dag}{\bf A}={\bf I}$. On the other hand, if ${\bf A}{\btheta}={\bf y}$ is an
under-determined system (i.e., $m<d$) and the rows of ${\bf A}$ are linearly independent,
then ${\bf A}^{\dag}={\bf A}^T({\bf A}{\bf A}^T)^{-1}$ and ${\bf A}{\bf A}^{\dag}={\bf
    I}$. Here, the solution to the system can be derived by substituting $\hat{\btheta}={\bf
    A}^{\dag}\hat{\balpha}$ (a projection of $d$ space onto the $m$ subspace) into the system
and then solve for $\hat{\balpha}\in\Real^m$ which results in $\hat{\balpha}={\bf y}$.
The result in Lemma~\ref{thm_LSapprox} is again obtained by substituting
$\hat{\balpha}={\bf y}$ into $\hat{\btheta}={\bf A}^{\dag}\hat{\balpha}$. In practice,
the full rank condition may not be easily achieved for data with large dimension without
regularization. For such cases, the Moore-Penrose inverse in
Definition~\ref{defn_pseudoinverse} exists and is unique. This means that for any matrix
${\bf A}$, there is precisely one matrix ${\bf A}^{\dag}$ that satisfies the four
properties of the Penrose equations in Definition~\ref{defn_pseudoinverse}. In general,
if ${\bf A}$ has a full rank factorization such that ${\bf A}={\bf F}{\bf G}$, then ${\bf
    X}^{\dag}={\bf G}^*({\bf G}{\bf G}^*)^{-1}({\bf F}^*{\bf F})^{-1}{\bf F}^*$ satisfies the Penrose
equations \cite{Adi1}.

The above algebraic relationships show that \emph{learning of a linear system in the
    sense of least squares error minimization can be achieved by manipulating its kernel or
    the range space via the Moore-Penrose inverse operation} (i.e., multiplying the
pseudoinverse of a system matrix to both sides of the equation boils down to an implicit
least squares error minimization, the readers are referred to \cite{Toh98} for greater
details regarding the least error and the least norm properties). For linear systems with
multiple ($q$) output columns, the following result is observed.

\begin{lemma} \label{lemma_LS_matrix} \cite{Toh97}
    Solving for $\bm{\Theta}$ in the system of linear equations of the form
    \begin{equation}\label{eqn_LinearEqn_matrix}
    {\bf A}\bm{\Theta} = {\bf Y}, \ \ \ {\bf A}\in\Real^{m\times d},\ \bm{\Theta}\in\Real^{d\times q},
    \ {\bf Y}\in\Real^{m\times q}
    \end{equation}
    in the column space (range) of ${\bf A}$ or in the row space (kernel) of ${\bf A}$ is
    equivalent to minimizing the sum of squared errors given by
    \begin{equation}\label{eqn_squared_error_distance_trace}
    \textup{SSE} = trace\left( ({\bf A}\bm{\Theta}-{\bf Y})^T({\bf A}\bm{\Theta}-{\bf Y}) \right).
    \end{equation}
    Moreover, the resultant solution $\hat{\bm{\Theta}}={\bf A}^{\dag}{\bf Y}$ is unique with a minimum-norm value
    in the sense that $\|\hat{\bm{\Theta}}\|^2_2\leq\|\bm{\Theta}\|^2_2$ for all feasible
    $\bm{\Theta}$.
\end{lemma}

\begin{proof}
    Equation \eqref{eqn_LinearEqn_matrix} can be re-written as a set of multiple linear
    systems as follows:
    \begin{equation}\label{eqn_LinearEqn_matrix_decomposed}
    {\bf A}[\btheta_1,\cdots,\btheta_q] = [{\bf y}_1,\cdots,{\bf y}_q].
    \end{equation}
    Since the trace of $({\bf A}\bm{\Theta}-{\bf Y})^T({\bf A}\bm{\Theta}-{\bf Y})$ is equal
    to the sum of the squared lengths of the error vectors ${\bf A}\btheta_i-{\bf y}_i$,
    $i=1,2,...,q$, the unique solution $\hat{\bm{\Theta}}=({\bf A}^T{\bf A})^{-1}{\bf
        A}^T{\bf Y}$ in the column space of ${\bf A}$ or that $\hat{\bm{\Theta}}={\bf A}^T({\bf
        A}{\bf A}^T)^{-1}{\bf Y}$ in the row space of ${\bf A}$, not only minimizes this sum, but
    also minimizes each term in the sum \cite{Duda1}. Moreover, since the column and the row
    spaces are independent, the sum of the individually minimized norms is also minimum.
\end{proof}

Based on these observations, the inherent least squares error approximation property of
algebraic manipulation utilizing the Moore-Penrose inverse shall be exploited in the
following section to derive an analytic solution for multilayer network learning that is
gradient-free.

\subsection{Feedforward Neural Network}

We consider a \emph{multilayer feedforward network} of $n$ layers. This network can
either be fully connected or with a \emph{receptive field} setting (also known as a
receptive field network, we abbreviated it as RFN hereon) as shown in
Fig.~\ref{fig_DN_NLayer}. When a full receptive field is set for all layers, the network
is a fully connected one \cite{Toh97,Toh98}. Unlike conventional networks, the bias term
in each layer is excluded except for the inputs in this representation.

Mathematically, the $n$-layer network model of $h_1$-$h_1$-$\cdots$-$h_{n-1}$-$h_n$
structure can be written in matrix form as
\begin{equation}\label{eqn_Nlayer_net1}
{\bf G} = \f_n\left(\f_{n-1}(\cdots\f_2(\f_1({\bf
    X}{\bf W}_1){\bf W}_2)\cdots{\bf W}_{n-1}){\bf W}_n \right),
\end{equation}
where ${\bf X}=[{\bf 1},\bm{X}]\in\Real^{m\times(d+1)}$ is the augmented input data
matrix (i.e., input with bias), ${\bf W}_1\in\Real^{(d+1)\times h_1}$, ${\bf
    W}_2\in\Real^{h_1\times h_2}$, $\cdots$, ${\bf W}_{n-1}\in\Real^{h_{n-2}\times h_{n-1}}$,
and ${\bf
    W}_{n}\in\Real^{h_{n-1}\times q}$ ($h_n=q$ is the output dimension) are the weight matrices without bias at each layer, and
${\bf G}\in\Real^{m\times q}$ is the network output of $q$ dimension. $f_k$,
$k=1,\cdots,n$ are activation functions which operate elementwise on its matrix domain.
When the $k^{th}$ layer of the network has a limited receptive field, then its weight
matrix has sparse elements in each column. For example, if the first layer has receptive
field of $3$ (denoted as $r_3$), then, when $h_1=d$, ${\bf W}_1$ can be written as a
circulant matrix \cite{Davis1,Golub1} as follows:
\begin{eqnarray}\label{eqn_weight_sparse1}
&& \hspace{2.3cm} h_1=d \nonumber \\
{\bf W}_1 &=& \overbrace{\left[\begin{array}{ccccc}
    w_{1,1} & 0       & 0       & \cdots  & 0 \\
    w_{2,1} & w_{2,2} & 0       & \cdots  & 0 \\
    w_{3,1} & w_{3,2} & w_{3,3} & \cdots  & 0 \\
    0       & w_{4,2} & w_{4,3} & \cdots  & 0 \\
    0       & 0       & w_{5,3} & \cdots  & 0 \\
    0       & 0       & 0       & \cdots  & 0 \\
    0       & 0       & 0       & \ddots  & w_{d-1,d}\\
    0       & 0       & 0       & \cdots  & w_{d,d}\\
    0       & 0       & 0       & \cdots  & w_{d+1,d}\\
    \end{array}\right]} .
\end{eqnarray}

For simplicity, the activation functions in each layer are taken to be the same (i.e.,
$f_k=f$, $k=1,\cdots,n$) in our development. Also, unless stated otherwise, all
activation functions are taken to operate elementwise on its matrix domain. We shall
denote this network with receptive field as one with structure
$h_1^{r_3}$-$h_2$-$\cdots$-$h_{n-1}$-$h_{n}$ which carries the common activation function
in each layer of $f(h_1^{r_3})$-$f(h_2)$-$\cdots$-$f(h_{n-1})$-$f(h_{n})$.

\begin{figure}[hhh]
    \begin{center}
        \epsfxsize=10.88cm
        \epsffile[64    48  1064   612]{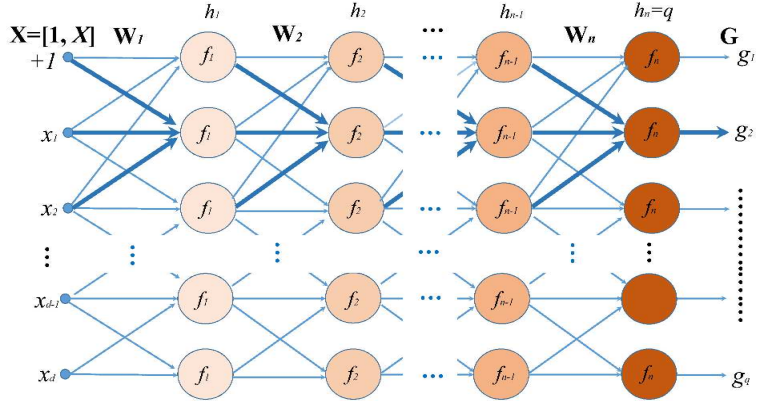}
        \caption{An $n$-layer feedforward network without inner weight biases. }
        \label{fig_DN_NLayer}
    \end{center}
\end{figure}

\subsection{Invertible Function}

In some circumstances, we may need to invert the network over the activation function for
solution seeking. For such a case, an inversion is performed through a functional
inversion. The inverse function is defined as follows.

\begin{defn}
    Consider a function $f$ which maps $x\in\Real$ to $y\in\Real$, i.e., $y=f(x)$. Then the
    inverse function for $f$ is such that $f^{-1}(y)=x$.
\end{defn}

A good example for invertible function is the \emph{softplus} function given by
$y=f(x)=\log(1+e^{x})$ and its inverse given by $\log(e^{y}-1)$. Although other
invertible functions can be used, this function and its modified versions, shall be
adopted in all our experiments for illustration purpose.

\section{Network Learning} \label{sec_net_learning}

The indexing of $f_k$, $k=1,\cdots,n$ for each layer of \eqref{eqn_Nlayer_net1} is kept
here for clarity purpose and suppose $f_k=f$ is an invertible function. Then, the network
can learn a given target matrix ${\bf Y}\in\Real^{m\times q}$ by putting ${\bf G}={\bf
    Y}$ where each layer of the network can be inverted (via taking functional inverse and
pseudoinverse) as follows:
\begin{eqnarray}
&& {\bf Y} = \f_n\left(\f_{n-1}(\f_{n-2}(\cdots\f_2(\f_1({\bf X}{\bf W}_1){\bf W}_2)
\cdots){\bf W}_{n-1}){\bf W}_n \right) \nonumber\\
&\Rightarrow& \g_{n}({\bf Y}) = \f_{n-1}(\f_{n-2}(\cdots\f_2(\f_1({\bf X}{\bf W}_1){\bf
W}_2)
\cdots){\bf W}_{n-1}){\bf W}_n \label{eqn_G1} \\
&\Rightarrow& \g_{n}({\bf Y}){\bf W}_n^{\dag} = \f_{n-1}(\f_{n-2}(\cdots\f_2(\f_1({\bf
X}{\bf W}_1){\bf W}_2)
\cdots){\bf W}_{n-1})  \nonumber\\
&\Rightarrow& \g_{n-1}({\g_{n}({\bf Y})\bf W}_n^{\dag}) = \f_{n-2}(\cdots\f_2(\f_1({\bf
X}{\bf W}_1){\bf W}_2)
\cdots){\bf W}_{n-1}  \label{eqn_G2} \\
&\Rightarrow& \vdots  \nonumber\\
&\Rightarrow& \g_{2}(\cdots\g_{n-1}(\g_{n}({\bf Y}){\bf W}_n^{\dag}){\bf
W}_{n-1}^{\dag}\cdots) =
\f_1({\bf X}{\bf W}_1){\bf W}_2  \label{eqn_G3} \\
&\Rightarrow& \g_{1}(\g_{2}(\cdots\g_{n-1}(\g_{n}({\bf Y}){\bf W}_n^{\dag}){\bf
W}_{n-1}^{\dag}\cdots){\bf W}_2^{\dag}) = {\bf X}{\bf W}_1 .\label{eqn_G4}
\end{eqnarray}
Based on \eqref{eqn_G1} to \eqref{eqn_G4}, we can express the weight matrices
respectively as
\begin{eqnarray}
{\bf W}_{1} &=& {\bf X}^{\dag}\g_{1}(\g_{2}(\g_{3}(\cdots\g_{n-1}(\g_{n}({\bf Y}){\bf
W}_n^{\dag}){\bf W}_{n-1}^{\dag}\cdots){\bf W}_3^{\dag}){\bf W}_2^{\dag})
\label{eqn_GW4} \\
{\bf W}_{2} &=& [\f_1({\bf X}{\bf W}_1)]^{\dag} \g_{2}(\g_{3}(\cdots\g_{n-1}(\g_{n}({\bf
Y}){\bf W}_n^{\dag}){\bf W}_{n-1}^{\dag}\cdots){\bf W}_3^{\dag})
\label{eqn_GW3} \\
&\vdots& \nonumber\\
{\bf W}_{n-1} &=& [\f_{n-2}(\cdots\f_2(\f_1({\bf X}{\bf W}_1){\bf W}_2) \cdots{\bf
W}_{n-2})]^{\dag} \g_{n-1}(\g_{n}({\bf Y}){\bf W}_n^{\dag})
\nonumber \\
\label{eqn_GW2} \\
{\bf W}_n &=& [\f_{n-1}(\f_{n-2}(\cdots\f_2(\f_1({\bf X}{\bf W}_1){\bf W}_2) \cdots){\bf
W}_{n-1})]^{\dag} \g_{n}({\bf Y}) .\label{eqn_GW1}
\end{eqnarray}

\noindent With these derivations, the follow result is stated formally.

\begin{theorem}
    Given $m$ data samples of $d$ dimension which are packed as ${\bf X}\in\Real^{m\times(d+1)}$ with augmentation, and consider a multilayer network of the form in \eqref{eqn_Nlayer_net1} with $f_k=f$,
    $\forall k=1,\cdots,n$ being invertible. Then, learning of the network towards the target
    ${\bf Y}\in\Real^{m\times q}$ in the sense of least squares error approximation boils
    down to solving the set of equations given by \eqref{eqn_GW4} through \eqref{eqn_GW1}.
\end{theorem}

\begin{proof}
    Since the Moore-Penrose inverse exists uniquely for any complex or real matrix and the
    activation functions are invertible, the results follow directly the derivations in
    \eqref{eqn_G1} through \eqref{eqn_GW4}. The assertion of least squares error
    approximation comes from the observation in Lemma~\ref{lemma_LS_matrix} where each step
    of manipulating each equation (in \eqref{eqn_G1} through \eqref{eqn_GW4}) by taking the
    pseudoinverse implies a projection onto either the column space or the row space.
    Mathematically, suppose ${\bf A}=\f_{n-1}(\f_{n-2}(\cdots\f_2(\f_1({\bf X}{\bf W}_1){\bf
        W}_2)\cdots){\bf W}_{n-1})$. Then \eqref{eqn_GW1} can be rewritten as ${\bf W}_n={\bf
        A}^{\dag}\g_n({\bf Y})$. The least squares approximation property can be easily
    visualized by substituting \eqref{eqn_GW1} into \eqref{eqn_Nlayer_net1} to get ${\bf
        G}=\f_n({\bf A}{\bf W}_n)=\f_n({\bf A}{\bf A}^{\dag}\g_n({\bf Y}))$. By taking $\g_n$ to
    both sides of the equation, we have $\g_n({\bf G})={\bf A}{\bf A}^{\dag}\g_n({\bf Y})$.
    Considering each column of the equation given by $\g_n({\bf g})={\bf A}{\bf
        A}^{\dag}\g_n({\bf y})$, we observe the analogue to $\hat{\bf y}={\bf
        A}\hat{\btheta}={\bf A}{\bf A}^{\dag}{\bf y}$ of Lemma~\ref{thm_LSapprox} by putting
    ${\bf g}=\hat{\bf y}$ where the invertible transform of the target $\g_n({\bf y})$ can be
    recovered uniquely. Hence the proof.
\end{proof}

The above result shows that the least squares solution to the network weights comes in a
cross-coupling form where the weight matrices are interdependent. This poses difficulty
towards an analytical estimation. Fortunately, due to the high redundancy of the network
weights that are constrained within an array structure, the solutions to network learning
are vastly available (see section~\ref{sec_depth} for a detailed analysis). The following
result is obtained by having the weight matrices ${\bf W}_k$, $k=2,\cdots,n$ arbitrarily
initialized as non-trivial matrices ${\bf R}_k$, $k=2,\cdots,n$ (e.g., a matrix which has
zero value for all its elements is considered a trivial matrix).

\begin{corollary} \label{cor_soln1}
    Consider $m$ data samples of $d$ dimension which are packed as ${\bf X}\in\Real^{m\times(d+1)}$ with augmentation.
    Given non-trivial matrices ${\bf R}_i$, $i=2,\cdots,n$ where their dimensions matches those of ${\bf
        W}_j$, $j=2,\cdots,n$ in \eqref{eqn_Nlayer_net1} and suppose the activation functions are invertible.
    Then, learning of the network towards the target ${\bf Y}\in\Real^{m\times q}$
    admits the following solution to the least squares error approximation:
    \begin{eqnarray}
    {\bf W}_{1} &=& {\bf X}^{\dag}\g_{1}(\g_{2}(\g_{3}(\cdots\g_{n-1}(\g_{n}({\bf Y}){\bf R}_n^{\dag}){\bf R}_{n-1}^{\dag}\cdots){\bf R}_3^{\dag}){\bf R}_2^{\dag})
    \label{eqn_R1} \\
    {\bf W}_{2} &=& [\f_1({\bf X}{\bf W}_1)]^{\dag}
    \g_{2}(\g_{3}(\cdots\g_{n-1}(\g_{n}({\bf Y}){\bf R}_n^{\dag}){\bf R}_{n-1}^{\dag}\cdots){\bf R}_3^{\dag})
    \label{eqn_R2} \\
    &\vdots& \nonumber\\
    {\bf W}_{n-1} &=& [\f_{n-2}(\cdots\f_2(\f_1({\bf X}{\bf W}_1){\bf W}_2)
    \cdots)]^{\dag} \g_{n-1}(\g_{n}({\bf Y}){\bf R}_n^{\dag})
    \label{eqn_R3} \\
    {\bf W}_n &=& [\f_{n-1}(\f_{n-2}(\cdots\f_2(\f_1({\bf X}{\bf W}_1){\bf W}_2)
    \cdots){\bf W}_{n-1})]^{\dag} \g_{n}({\bf Y})
    .\label{eqn_R4}
    \end{eqnarray}
\end{corollary}

\begin{proof}
    By initializing the weights ${\bf W}_j$, $j=2,\cdots,n$ using matrices ${\bf R}_i$,
    $i=2,\cdots,n$, the inter-dependency of weights can be decoupled and the weights can be
    estimated sequentially from \eqref{eqn_R1} to \eqref{eqn_R4}. Hence the proof.
\end{proof}

\begin{corollary} \label{cor_soln2}
    Consider $m$ data samples of $d$ dimension which are packed as ${\bf X}\in\Real^{m\times(d+1)}$ with augmentation.
    Given non-trivial matrices ${\bf R}_i$, $i=1,2,4,5\cdots,n$ where their dimensions matches those of
    ${\bf W}_j$, $j=1,2,4,5\cdots,n$ in \eqref{eqn_Nlayer_net1} and suppose the activation functions are invertible. Then, learning of the network towards the target ${\bf Y}\in\Real^{m\times q}$ admits the
    following solution to the least squares error approximation:
    \begin{eqnarray}
    {\bf W}_{3} &=& [\f_2(\f_1({\bf X}{\bf R}_1){\bf R}_2)]^{\dag}
    \g_{3}(\g_{4}(\g_{5}(\cdots\g_{n-1}(\g_{n}({\bf Y}){\bf R}_n^{\dag}){\bf R}_{n-1}^{\dag}\cdots){\bf R}_5^{\dag}){\bf R}_4^{\dag})
    \nonumber \\
    \label{eqn_RR3} \\
    {\bf W}_{4} &=& [\f_3(\f_2(\f_1({\bf X}{\bf R}_1){\bf R}_2){\bf W}_3)]^{\dag}
    \g_{4}(\g_{5}(\cdots\g_{n-1}(\g_{n}({\bf Y}){\bf R}_n^{\dag}){\bf R}_{n-1}^{\dag}\cdots){\bf R}_5^{\dag})
    \nonumber\\ \label{eqn_RR4} \\
    &\vdots& \nonumber \\
    {\bf W}_{n-1} &=& [\f_{n-2}(\cdots\f_3(\f_2(\f_1({\bf X}{\bf R}_1){\bf R}_2)
    {\bf W}_3)\cdots{\bf W}_{n-2})]^{\dag} \g_{n-1}(\g_{n}({\bf Y}){\bf R}_n^{\dag})
    \label{eqn_RRn1} \\
    {\bf W}_{1} &=& {\bf X}^{\dag}\g_{1}(\g_{2}(\g_{3}(\cdots\g_{n-1}(\g_{n}({\bf Y}){\bf R}_n^{\dag}){\bf W}_{n-1}^{\dag}\cdots){\bf W}_3^{\dag}){\bf R}_2^{\dag})
    \label{eqn_RR1} \\
    {\bf W}_{2} &=& [\f_1({\bf X}{\bf W}_1)]^{\dag}
    \g_{2}(\g_{3}(\cdots\g_{n-1}(\g_{n}({\bf Y}){\bf R}_n^{\dag}){\bf W}_{n-1}^{\dag}\cdots){\bf W}_3^{\dag})
    \label{eqn_RR2} \\
    {\bf W}_n &=& [\f_{n-1}(\f_{n-2}(\cdots\f_3(\f_2(\f_1({\bf X}{\bf W}_1){\bf W}_2){\bf W}_3)
    \cdots){\bf W}_{n-1})]^{\dag} \g_{n}({\bf Y})
    .\label{eqn_RRn}
    \end{eqnarray}
\end{corollary}

\begin{proof}
    The order of learning sequence is immaterial for any of the inner layer weights ${\bf
        W}_i$, $i=1,2,4,\cdots,n-1$, but the output layer weight ${\bf W}_n$ should be lastly
    estimated. This is because ${\bf W}_n$ is the weight which links up all hidden layer
    weights and provides a direct least squares error approximation (range space projection)
    to the target ${\bf Y}$. In other words, the estimation of ${\bf W}_n$ \eqref{eqn_RRn} at
    the output layer is complete when all the hidden weights have been determined. Hence the
    proof.
\end{proof}

Apart from the above solution based on random initialization, the following result shows
an example of non-random initialization. This is particularly useful when the activation
functions are non-invertible.

\begin{corollary} \label{cor_soln3}
    Given $m$ data samples of $d$ dimension which are packed as ${\bf X}\in\Real^{m\times(d+1)}$ with augmentation.
    Then, learning of the network \eqref{eqn_Nlayer_net1} towards the target ${\bf Y}\in\Real^{m\times q}$ admits the following solution to the least squares error approximation:
    \begin{eqnarray}
    {\bf W}_{1} &=& {\bf X}^{\dag}
    \label{eqn_C3R1} \\
    {\bf W}_{2} &=& [\f_1({\bf X}{\bf W}_1)]^{\dag}
    \label{eqn_C3R2} \\
    &\vdots& \nonumber\\
    {\bf W}_{n-1} &=& [\f_{n-2}(\cdots\f_2(\f_1({\bf X}{\bf W}_1){\bf W}_2)
    \cdots)]^{\dag}
    \label{eqn_C3R3} \\
    {\bf W}_n &=& [\f_{n-1}(\f_{n-2}(\cdots\f_2(\f_1({\bf X}{\bf W}_1){\bf W}_2)
    \cdots){\bf W}_{n-1})]^{\dag} \g_{n}({\bf Y})
    .\label{eqn_C3R4}
    \end{eqnarray}
\end{corollary}

\begin{proof}
    Since ${\bf R}_i$, $i=2,\cdots,n$ in Corollary~\ref{cor_soln1} are arbitrarily assigned
    matrices that matches the dimensions of the corresponding weight matrices, they, together
    with $\g_i$, $i=1,\cdots,n$ on ${\bf Y}$ are indeed arbitrary, and they all can be chosen
    as identity matrix as one particular solution. Since the estimation of the output weight
    matrix ${\bf W}_n$ in \eqref{eqn_R4} of Corollary~\ref{cor_soln1} does not involve
    assignment of the arbitrary matrix ${\bf R}_i$, the transformed target term $\g_{n}({\bf
        Y})$ cannot be replaced by the identity matrix. Hence the proof.
\end{proof}

\noindent{\bf Remark 1: } The result in Corollary~\ref{cor_soln1} shows that there exits
an infinite number of solutions for a network with multiple layers due to the arbitrary
weight initialization. However, this does not apply to the single layer network because
it can be reduced to the system of linear equations. The result in
Corollary~\ref{cor_soln2} shows that the sequence of weights estimation need not begin
from the first layer. Indeed, the arbitrary weights in Corollary~\ref{cor_soln1} and
Corollary~\ref{cor_soln2} are only utilized once following the sequence of weight matrix
estimation. The weight matrix after each step of the estimation sequence is no longer
arbitrary because all the previously estimated weights are used in the subsequent
pseudoinverse projection towards the least squared error approximation.
Corollary~\ref{cor_soln3} shows a particular choice of the arbitrary weight matrix
assignment without needing the corresponding activation functions to be invertible. For
non-invertible functions, the term $\g_{n}({\bf Y})$ in \eqref{eqn_C3R4} can be replaced
by ${\bf Y}$ which corresponds to a linear output layer in \eqref{eqn_Nlayer_net1}.
However, such an identity setting results in having the number of network nodes in every
hidden layer (i.e., $h_i$, the column size of ${\bf W}_i$, $i=1,\cdots,n-1$) being set
equal to the sample size $m$. \Box

The proposed algorithms in Corollary~\ref{cor_soln1} through Corollary~\ref{cor_soln3}
are differentiated from that in \cite{Baldi1} in several ways. Firstly, the nonlinear
activation function was not considered in \cite{Baldi1} whereas in our approach,
nonlinear activations, both invertible and not, are considered. Secondly, in
Corollary~\ref{cor_soln1} through Corollary~\ref{cor_soln3}, there is no attempt to
alternate the search of weights among the layers as that in \cite{Baldi1} since such a
search leads to the convergence issue. Lastly but not the least, we capitalize on the
powerful Moore-Penrose inverse for solving the weights. More importantly, the layered
network learning solution becomes analytical after arriving at \eqref{eqn_C3R4} (and so
are \eqref{eqn_R4} and \eqref{eqn_RRn}), i.e., the problem boils down to finding an input
correlated matrix $\f_{n-1}(\f_{n-2}(\cdots\f_2(\f_1({\bf X}{\bf W}_1){\bf
W}_2)\cdots){\bf W}_{n-1})$ that can represent ${\bf Y}$. This can be seen by considering
a network with linear output where ${\bf G} = f_{n-1}({\bf A}_{n-1}){\bf W}_{n}$ with
${\bf A}_{n-1}=f_{n-2}({\bf A}_{n-2}){\bf
    W}_{n-1}$, $\cdots$, ${\bf A}_1={\bf X}{\bf W}_1$. When the learning is based on
Corollary~\ref{cor_soln3} where ${\bf W}_1={\bf X}^{\dag}$, ${\bf W}_2=[\f_{1}({\bf
    A}_{1})]^{\dag}$, $\cdots$, ${\bf W}_n=[\f_{n-1}({\bf A}_{n-1})]^{\dag}{\bf Y}$, then we
have the estimated output given by
\begin{equation}\label{eqn_Image_Y}
\hat{\bf Y} = \f_{n-1}({\bf A}_{n-1})[\f_{n-1}({\bf A}_{n-1})]^{\dag}{\bf Y},
\end{equation}
where ${\bf A}_{n-1}=\f_{n-2}({\bf A}_{n-2})[\f_{n-2}({\bf A}_{n-2})]^{\dag}$, $\cdots$,
${\bf A}_1={\bf X}{\bf X}^{\dag}$. This shows that each layer in \eqref{eqn_Image_Y}
serves as a composition of projection bases, based on the data matrix or the transformed
data matrix (transformed by $f_k$, $k=1,\cdots,n-1$), to hold the target ${\bf Y}$ within
its range. In a nutshell, our results show that \emph{the problem of network learning can
    be viewed as one to find a data mapping matrix such that the target matrix falls within
    its range. An immediate advantage of such a view over the conventional view of error
    minimization is evident from the gradient-free solutions arrived in
    Corollaries~\ref{cor_soln1}--\ref{cor_soln3}. } These results shall be validated by
numerical evaluations in the experimental section.


\section{Network Representation and Generalization} \label{sec_net_representation}

Firstly, we show the network representation capability. Then, an analysis of the feasible
solution space is presented. Finally, an analysis of the output variance is performed.

\subsection{Network Representation}

For simplicity, we shall use $d$ as the dimension for generic matrices instead of the
augmented dimension $d+1$ when no ambiguity arises.

\begin{defn} \label{defn_representative}
    A matrix ${\bf A}\in\Real^{m\times d}$ is said to be representative for ${\bf Y}\in\Real^{m\times q}$
    if ${\bf A}{\bf A}^{\dag}{\bf Y}={\bf Y}$
    where $m,d,q\geq 1$.
\end{defn}

Certainly, if ${\bf A}$ has full rank in the sense that ${\bf A}{\bf A}^T$ is invertible
for ${\bf A}^{\dag}={\bf A}^T({\bf A}{\bf A}^T)^{-1}$, then ${\bf A}{\bf A}^{\dag}={\bf
    I}$ for which ${\bf A}$ is representative  for ${\bf Y}$. However, this definition of
representation is weaker than the full rank requirement in $({\bf A}^T{\bf A})^{-1}{\bf
    A}^T$ or ${\bf A}^T({\bf A}{\bf A}^T)^{-1}$  because ${\bf A}$ needs not have full rank
when ${\bf Y}$ falls within the range of ${\bf A}$.

\begin{defn} \label{defn_representative_f}
    A function $f$, which operates elementwise on its matrix domain, is said to be representative if $f({\bf X}{\bf W}_1)$ is representative
    for ${\bf Y}\in\Real^{m\times q}$ on ${\bf X}\in\Real^{m\times d}$ with distinct rows and
    ${\bf W}_1\in\Real^{d\times h_1}$ where $m,d,h_1,q\geq 1$, i.e., $f$ is a function such that
    $f({\bf X}{\bf W}_1)f({\bf X}{\bf W}_1)^{\dag}{\bf Y}={\bf Y}$.
\end{defn}

The above definition can accommodate the identity activation function (i.e., $f(x)=x$) at
certain circumstance because $({\bf X}{\bf W}_1)({\bf X}{\bf W}_1)^{\dag}{\bf Y}={\bf Y}$
when the column rank of the matrix ${\bf X}{\bf W}_1$ matches the row size of ${\bf Y}$.
However, such an identity activation does not provide sufficient representation
capability for mapping of targets beyond the range space of the data matrix. For common
activation functions such as the \texttt{softplus} and the \texttt{tangent} functions, it
is observed that imposing them on ${\bf X}{\bf W}_1$ can improve the column rank
conditioning for representation in practice. Several numerical examples are given after
the theoretical results to observe the representation capability.

\begin{theorem} (\textbf{Two-layer Network}) \label{thm_twolayer_representation}
    Given $m$ distinct samples of input-output data pairs and suppose the activation
    functions are representative according to Definition~\ref{defn_representative_f}.
    Then there exists a feedforward network of two
    layers with at least $m\times q$ adjustable weights that can
    represent these data samples of $q\geq 1$ output dimensions.
\end{theorem}

\begin{proof}
    Consider a 2-layer network, with linear activation at the output layer, that takes an augmented set of inputs:
    \begin{equation}\label{eqn_2layer_net}
    {\bf G} = f({\bf X}{\bf W}_1){\bf W}_2,
    \end{equation}
    where ${\bf X}\in\Real^{m\times(d+1)}$, ${\bf
        W}_1\in\Real^{(d+1)\times h_1}$, ${\bf W}_2\in\Real^{h_1\times q}$ and
    ${\bf G}\in\Real^{m\times q}$ is the network output of $q\geq 1$ dimensions. Then,
    ${\bf G}$ \eqref{eqn_2layer_net} can be used to learn the target matrix ${\bf Y}\in\Real^{m\times q}$ by putting
    \begin{equation}\label{eqn_2layer_net_v2}
    {\bf Y} = f({\bf X}{\bf W}_1){\bf W}_2.
    \end{equation}
    Since $f$ is given to be a representative function according to
    Definition~\ref{defn_representative_f},
    we have\\ $f({\bf X}{\bf W}_1)f({\bf X}{\bf W}_1)^{\dag}
    {\bf Y}$ $={\bf Y}$ such that ${\bf Y}$ falls within the range space of
    $f({\bf X}{\bf W}_1)$. To ensure that ${\bf Y}$ falls within the range space of
    $f({\bf X}{\bf W}_1)$, it is sufficient that ${\bf W}_1$ having at least a column size
    of $h_1=m$ in order to match with the number of samples in ${\bf Y}$. Suppose we have
    a non-trivial setting for all the weight elements in ${\bf W}_1$ such that
    not all of its elements are zeros. A good example for the non-trivial setting is
    to put ${\bf W}_1={\bf X}^{\dag}$ (see Corollary~\ref{cor_soln3}). Then, due to the representativeness of $f$, the entire weight matrix ${\bf W}_2\in\Real^{m\times q}$ can be
    determined uniquely as $f({\bf X}{\bf W}_1)^{\dag}
    {\bf Y}$ to represent the target ${\bf Y}$. In other words,
    ${\bf W}_1$ needs only distinct and arbitrary numbers in its $m$ columns of entries to maintain the representation
    sufficiency of $f({\bf X}{\bf W}_1)$, and all those elements in ${\bf W}_2$
    are the only adjustable parameters needed for the representation. Hence the proof.
\end{proof}

\noindent{\bf Remark 2: } The above result shows that the two-layer network is a
universal approximator for a finite set of data. Different from most existing results in
literature \cite{Funa1,Hornik1,Cybenko1,HechNiel3,Kurkova2,Leshno1}, the minimum number
of adjustable hidden weights required here is dependent on the number of data samples and
the output dimension. This result also stretches beyond that of \cite{ZhangCY1} to
include nonlinear activation functions. \Box

\begin{defn} \label{defn_representative_f_composite}
    A function $f$ is said to be compositional representative if\\
    $f(\cdots f({\bf X}{\bf W}_1)\cdots{\bf W}_n)$ is representative for
    ${\bf Y}\in\Real^{m\times q}$ on ${\bf X}\in\Real^{m\times d}$ with distinct rows and
    ${\bf W}_k\in\Real^{h_{k-1}\times h_k}$, $k=1,\cdots,n$ where $h_0=d+1$, $m,d,h_k,q\geq 1$,
    i.e., $f$ is a function such that $f(\cdots f({\bf X}{\bf W}_1)\cdots{\bf W}_n)
    f(\cdots f({\bf X}{\bf W}_1)\cdots{\bf W}_n)^{\dag}{\bf Y}={\bf Y}$.
\end{defn}

\begin{theorem} (\textbf{Multilayer Network}) \label{thm_fitting_multilayer}
    Given $m$ distinct samples of input-output data pairs and suppose the
    activation functions are representative according to
    Definition~\ref{defn_representative_f_composite}. Then there exists a
    feedforward network of $n$ layers with 
    at least $m\times q$ adjustable weights that can represent these data samples
    samples of $q\geq 1$ output dimensions.
\end{theorem}

\begin{proof}
    Consider a feedforward network with linear activation at the output layer given by
    \begin{equation}\label{eqn_Nlayer_outermost}
    {\bf G} = f_{n-1}({\bf A}_{n-1}){\bf W}_{n} ,
    \end{equation}
    where ${\bf A}_{n-1}=f_{n-2}({\bf A}_{n-2}){\bf W}_{n-1}$,
    $\cdots$, ${\bf A}_1={\bf X}{\bf W}_1$. Learning towards ${\bf Y}$ using this network is analogous to
    \eqref{eqn_2layer_net_v2} in Theorem~\ref{thm_twolayer_representation} where we
    require $m\times q$ number of adjustable weight elements in ${\bf W}_n$ here for
    the desired representation. The representation property of each ${\bf A}_i$,
    $i=1,\cdots,n-1$ shall not change if each ${\bf W}_i$, $i=1,\cdots,n-1$ is a
    representative matrix with sufficient column size (such as \eqref{eqn_C3R1}--\eqref{eqn_C3R3}) and if $f$
    is compositional representative according to
    Definition~\ref{defn_representative_f_composite}.
    Since ${\bf W}_i$, $i=1,\cdots,n-1$ can be arbitrarily prefixed,
    the elements of ${\bf W}_n\in\Real^{m\times q}$ are the only adjustable
    parameters needed in learning the representation. This completes the proof.
\end{proof}

\subsection{Layer Depth and the Feasible Solution Space} \label{sec_depth}

Consider the learning problem of a single-layer network given by ${\bf Y}=\f_1({\bf
    X}{\bf W}_1)$. Effectively, such a network learning boils down to the linear regression
problem since it can be rewritten as $\g_1({\bf Y})={\bf X}{\bf W}_1$ when $\f_1$ is
invertible. Because the pseudoinverse is unique, the solution in the least squares error
approximation sense for ${\bf W}_1$ given by ${\bf X}^{\dag}\g_1({\bf Y})$ is unique.

For the two-layer network learning given by ${\bf Y}=\f_2(\f_1({\bf X}{\bf W}_1){\bf
    W}_2)$, the number of feasible solutions increases tremendously. This can be seen from
the re-written systems of linear equations given by $\g_2({\bf Y})=\f_1({\bf X}{\bf
    W}_1){\bf W}_2$ and $\g_1(\g_2({\bf Y}){\bf W}_2^{\dag})={\bf X}{\bf W}_1$ where fixing
arbitrarily either ${\bf W}_1$ or ${\bf W}_2$ determines the other uniquely. The vast
possibilities of either of the weight matrices determine the number of feasible
solutions.

For the three-layer network learning given by ${\bf Y}=\f_3(\f_2(\f_1({\bf X}{\bf
    W}_1){\bf W}_2){\bf W}_3)$, the feasible solution space increases further. Again, this
can be seen from the re-written systems of linear equations given by $\g_3({\bf
    Y})=\f_2(\f_1({\bf X}{\bf W}_1){\bf W}_2){\bf W}_3$, $\g_2(\g_3({\bf Y}){\bf
    W}_3^{\dag})=\f_1({\bf X}{\bf W}_1){\bf W}_2$ and $\g_1(\g_2(\g_3({\bf Y}){\bf
    W}_3^{\dag}){\bf W}_2^{\dag})={\bf X}{\bf W}_1$ where fixing arbitrarily ${\bf W}_1$ and
${\bf W}_2$ determines ${\bf W}_3$ (and so on for other combinations) uniquely. The vast
possibilities of the combination of the weight matrices determine the number of feasible
solutions. Based on this analysis, we gather that the number of feasible solutions
increases exponentially according to the layer depth.

\begin{proposition} \label{thm_soln_space}
    The number of feasible solutions for a feedforward network with two layers
    is infinite. This number increases exponentially for each increment of network layer.
\end{proposition}

\begin{proof}
    For a two-layer network, the number of feasible solutions is determined by the size of
    either ${\bf W}_1$ or ${\bf W}_2$ which is infinite in the real domain. Let us denote
    this number by $N$. Then, for a three-layer network, this number increases to $(N\times
    N)\times C^3_2=N^2\times 3$ since two arbitrary weights out of ${\bf W}_i$, $i=1,2,3$
    determine the remaining. For a $n$-layer network, we have $N^{n-1}\times C^n_{n-1}$
    possibilities. Hence the proof.
\end{proof}

\subsection{Layer Depth and the Output Variance, Bias}

For simplicity, consider the linear model ${g}=\bm{x}^T\bm{w}$ without the bias term
where $\bm{x}\in\Real^{d\times 1}$, $\bm{w}\in\Real^{d\times 1}$ regressing towards the
target $y\in\Real$ so that ${y}=\bm{x}^T\bm{w}+{\epsilon}$ where ${\epsilon}\in\Real$.
The expected prediction error at an input point $\bm{x}_0\in\Real^{d\times 1}$ using the
squared-error loss is (see e.g., \cite{Hastie01,Duda1}):
\begin{eqnarray}
\textrm{Err}(\bm{x}_0) &=& E[({y} - \hat{g}(\bm{x}_0))^2|\bm{x}=\bm{x}_0] \nonumber\\
&=& \sigma_{\epsilon}^2 + [E[\hat{g}(\bm{x}_0)] - {g}(\bm{x}_0)]^2
+ E[\hat{g}(\bm{x}_0)-E[\hat{g}(\bm{x}_0)]]^2 \nonumber\\
&=& \textrm{Irreducible\ Error} + \textrm{Bias}^2 + \textrm{Variance}.
\label{eqn_bias_variance}
\end{eqnarray}

To analyze the bias and variance terms of our network, we consider the result in
Corollary~\ref{cor_soln3} with only deterministic components in the estimation of weights
for simplicity. The estimation Bias from $k$ number of test samples based on
\eqref{eqn_bias_variance} is given by
\begin{eqnarray}
\textrm{Bias}^2 &=& \left[E[\hat{g}(\bm{x}_0)] - {g}(\bm{x}_0)\right]^2 \nonumber\\
&=& \frac{1}{k} \sum_{i=1}^{k} \left[ E[\hat{g}(\bm{x}_i)] - {g}(\bm{x}_i) \right]^2 \nonumber\\
&=& \frac{1}{k} \sum_{i=1}^{k} \left[ E[\bm{x}_i^T\hat{\bm{w}}_n] - {g}(\bm{x}_i) \right]^2 \nonumber\\
&=& \frac{1}{k} \sum_{i=1}^{k} \left[ E[\bm{x}_i^T[\f_{n-1}(\f_{n-2}(\cdots\f_2(\f_1({\bf
    X}\hat{\bf W}_1)\hat{\bf W}_2) \cdots)\hat{\bf W}_{n-1})]^{\dag} {\bf y}] - {g}(\bm{x}_i)
\right]^2.\nonumber\\
\end{eqnarray}
This shows that the estimation bias for unseen samples is dependent on the difference
between the estimated parameters and that of the unknown `true' parameters which maps the
range space of the unseen target. Since the `true' parameters are unknown and cannot be
removed, the bias term cannot be analyzed further.

By considering the linear regression estimate $\hat{\bm{w}}=({\bf X}^T{\bf X})^{-1}{\bf
    X}^T{\bf y}$ based on a set of training data ${\bf X}\in\Real^{m\times d}$ with target
given by ${\bf y}={\bf X}\bm{w}+\bm{\epsilon}$ for a certain unknown `true' parameter
$\bm{w}$ with error $\bm{\epsilon}\in\Real^{m\times 1}$, the Variance term of linear
regression can nevertheless be analyzed using (see e.g., \cite{Hastie01}):
\begin{eqnarray}
E[\hat{g}(\bm{x}_0)-E[\hat{g}(\bm{x}_0)]]^2
&=& E[\bm{x}_0^T\hat{\bm{w}} - \bm{x}_0^T\bm{w}]^2 \nonumber\\
&=& E[\bm{x}_0^T({\bf X}^T{\bf X})^{-1}{\bf X}^T{\bf y} - \bm{x}_0^T\bm{w}]^2 \nonumber\\
&=& E[\bm{x}_0^T({\bf X}^T{\bf X})^{-1}{\bf X}^T({\bf X}\bm{w}+\bm{\epsilon}) - \bm{x}_0^T\bm{w}]^2 \nonumber\\
&=& E[\bm{x}_0^T\bm{w} + \bm{x}_0^T({\bf X}^T{\bf X})^{-1}{\bf X}^T\bm{\epsilon} - \bm{x}_0^T\bm{w}]^2 \nonumber\\
&=& E[\bm{x}_0^T({\bf X}^T{\bf X})^{-1}{\bf X}^T\bm{\epsilon} ]^2 \nonumber\\
&=& \|{\bf h}_1(\bm{x}_0)\|^2\sigma^2_{\bm{\epsilon}}. \label{eqn_h1}
\end{eqnarray}
where ${\bf h}_1(\bm{x}_0)=\bm{x}_0^T({\bf X}^T{\bf X})^{-1}{\bf
    X}^T=\bm{x}_0^T{\bf X}^{\dag}$.

For the single-layer network with single output and linear activation, the estimation is
given by $ \hat{\bm{w}}_{1} = {\bf X}^{\dag}{\bf y} $ and the estimated output is given
by $ \hat{\bf y} = {\bf X}\hat{\bm{w}}_{1} = {\bf X}{\bf X}^{\dag}{\bf y}. $ This boils
down to the case of linear regression where the estimated $\hat{\bf y}$ varies according
to ${\bf h}_1(\bm{x}_0)\bm{\epsilon}=\bm{x}_0^T{\bf X}^{\dag}\bm{\epsilon}$ given by
\eqref{eqn_h1}.

Next, consider the two-layer network with linear output where the estimation is given by
\begin{eqnarray}
\hat{\bm{w}}_{1} &=& {\bf X}^{\dag},
\label{eqn_TVR1} \nonumber\\
\hat{\bm{w}}_{2} &=& [\f_1({\bf X}\hat{\bm{w}}_{1})]^{\dag}{\bf y}. \label{eqn_TVR2}
\end{eqnarray}
Then, the estimated output can be written as
\begin{eqnarray}
\hat{\bf y} &=& \f_1({\bf X}\hat{\bm{w}}_{1})\hat{\bm{w}}_{2} \nonumber\\
\hat{\bf y} &=& \f_1({\bf X}\hat{\bm{w}}_{1})[\f_1({\bf X}\hat{\bm{w}}_1)]^{\dag}{\bf y}\nonumber\\
\hat{\bf y} &=& \f_1({\bf X}{\bf X}^{\dag}) [\f_1({\bf X}{\bf X}^{\dag})]^{\dag}{\bf y} .
\label{eqn_TYVR1}
\end{eqnarray}
Since $\f_1({\bf X}\hat{\bm{w}}_{1})$ is assumed to be representative given the
activation function and the data matrix, the target ${\bf y}$ can be reproduced exactly
for the training set. Here, we note that the underlying estimation is based on
$\hat{\bm{w}}_{2}$, and this causes the output variance vector, ${\bf
    h}_2(\textbf{\textsf{x}}_0)$ where $\textbf{\textsf{x}}_0\in\Real^{m}$, to hinge upon the
feature dimension of $\hat{\bm{w}}_{2}\in\Real^{m}$ in
\begin{eqnarray}
{\bf h}_2(\textbf{\textsf{x}}_0)\bm{\epsilon} &=& \textbf{\textsf{x}}_0^T[\f_1({\bf
X}\hat{\bm{w}}_1)]^{\dag}\bm{\epsilon}
\nonumber \\
&=& \textbf{\textsf{x}}_0^T[\f_1({\bf X}{\bf X}^{\dag})]^{\dag}\bm{\epsilon}.
\end{eqnarray}
This is differentiated from the case of linear regression where $\bm{x}_0\in\Real^{d}$.
For applications with a large data set, $m>>d$, and this renders $({\bf
    h}_2(\textbf{\textsf{x}}_0)\bm{\epsilon})^2 >> ({\bf h}_1(\bm{x}_0)\bm{\epsilon})^2$.

For the three-layer network, the estimation is given by
\begin{eqnarray}
\hat{\bm{w}}_{1} &=& {\bf X}^{\dag},
\label{eqn_TTVR1} \nonumber\\
\hat{\bm{w}}_{2} &=& [\f_1({\bf X}\hat{\bm{w}}_1)]^{\dag},
\label{eqn_TTVR2} \nonumber\\
\hat{\bm{w}}_{3} &=& [\f_2(\f_1({\bf X}\hat{\bm{w}}_1)\hat{\bm{w}}_2)]^{\dag}{\bf y},
\label{eqn_TTVR3}
\end{eqnarray}
and the estimated output can be written as
\begin{eqnarray}
\hat{\bf y} &=& \f_2\left(\f_1({\bf X}\hat{\bm{w}}_1)\hat{\bm{w}}_2\right) \hat{\bm{w}}_3
.\label{eqn_TYVR1_h}
\end{eqnarray}
The corresponding output variance term for the three-layer network model is
\begin{eqnarray}
{\bf h}_3(\textbf{\textsf{x}}_0)\bm{\epsilon} &=& \textbf{\textsf{x}}_0^T[\f_2(\f_1({\bf
X}\hat{\bm{w}}_1)\hat{\bm{w}}_2)]^{\dag}\bm{\epsilon}
\nonumber \\
&=&  \textbf{\textsf{x}}_0^T[\f_2(\f_1({\bf X}{\bf X}^{\dag}) [\f_1({\bf X}{\bf
X}^{\dag})]^{\dag})]^{\dag}\bm{\epsilon}. \label{eqn_3layer_h}
\end{eqnarray}

Based on this analysis, we can generalize the output variance from the 1-layer network to
the $n$-layer network as:
\begin{equation}\label{eqn_h_gen}
{\bf h}_k(\textbf{\textsf{x}}_0)\bm{\epsilon} = \textbf{\textsf{x}}_0^T {\bf
H}_{k}^{\dag} \bm{\epsilon}, \ \ \ k=1,2,...,n,
\end{equation}
where
\begin{eqnarray}
{\bf H}_{1} &=& {\bf X},
\label{eqn_n2layer_h1_func} \\
{\bf H}_{2} &=& \f_1({\bf H}_1{\bf H}_1^{\dag}),
\label{eqn_n2layer_h2_func} \\
{\bf H}_{3} &=& \f_2({\bf H}_2{\bf H}_2^{\dag}),
\label{eqn_n3layer_h3_func} \\
&\vdots& \nonumber \\
{\bf H}_n &=& \f_{n-1}({\bf H}_{n-1}{\bf H}_{n-1}^{\dag}). \label{eqn_nlayer_hn_func}
\end{eqnarray}

\begin{proposition} \label{thm_output_variance}
    If $\f_{k}=f$, $k=1,\cdots,n-1$ is a function such that
    $E[({\bf h}_{2}(\textbf{\textsf{x}}_0)\bm{\epsilon})^2]$
    $>$ $E[({\bf h}_{3}(\textbf{\textsf{x}}_0)\bm{\epsilon})^2]$ for all $\textbf{\textsf{x}}_0\in\Real^m$
    and $\bm{\epsilon}\in\Real^m$, then
    $E[({\bf h}_{k}(\textbf{\textsf{x}}_0)\bm{\epsilon})^2]$ $>$
    $E[({\bf h}_{k+1}(\textbf{\textsf{x}}_0)\bm{\epsilon})^2]$,
    \ $\forall\ k=3,...,n$ in \eqref{eqn_h_gen}.
\end{proposition}

\begin{proof} From
    $E[(\textbf{\textsf{x}}_0^T[\f_1({\bf X}{\bf X}^{\dag})]^{\dag}\bm{\epsilon})^2]$
    $>$ $E[(\textbf{\textsf{x}}_0^T[\f_2(\f_1({\bf X}{\bf X}^{\dag})
    [\f_1({\bf X}{\bf X}^{\dag})]^{\dag})]^{\dag}\bm{\epsilon})^2]$,
    we observe that the term
    ${\bf X}$ has been replaced by $\f_1({\bf X}{\bf X}^{\dag})$. Such replacement is
    recursively performed for each additional layer of the activation function in deeper
    networks. Since the activation function is such that $E[({\bf
        h}_{2}(\textbf{\textsf{x}}_0)\bm{\epsilon})^2]$
    $>$ $E[({\bf h}_{3}(\textbf{\textsf{x}}_0)\bm{\epsilon})^2]$ for all $\textbf{\textsf{x}}_0\in\Real^m$,
    the inequality maintains for each replacement of the term ${\bf X}$ within. In other
    words, an addition of an inner layer shall not change the inequality. Hence the proof.
\end{proof}

\begin{conject} \label{conj_output_variance}
    If $\f_{k}=f$, $k=1,\cdots,n-1$ is a representative function,
    then $E[({\bf h}_{k}(\textbf{\textsf{x}}_0)\bm{\epsilon})^2]$ $>$ $E[({\bf h}_{k+1}(\textbf{\textsf{x}}_0)\bm{\epsilon})^2]$
    for all $k=2,...,n$ in \eqref{eqn_h_gen}.
\end{conject}

\begin{figure}[hhh]
    \begin{center}
        \epsfxsize=12cm
        \epsfysize=5.68cm
        \epsffile[86    17   872   427]{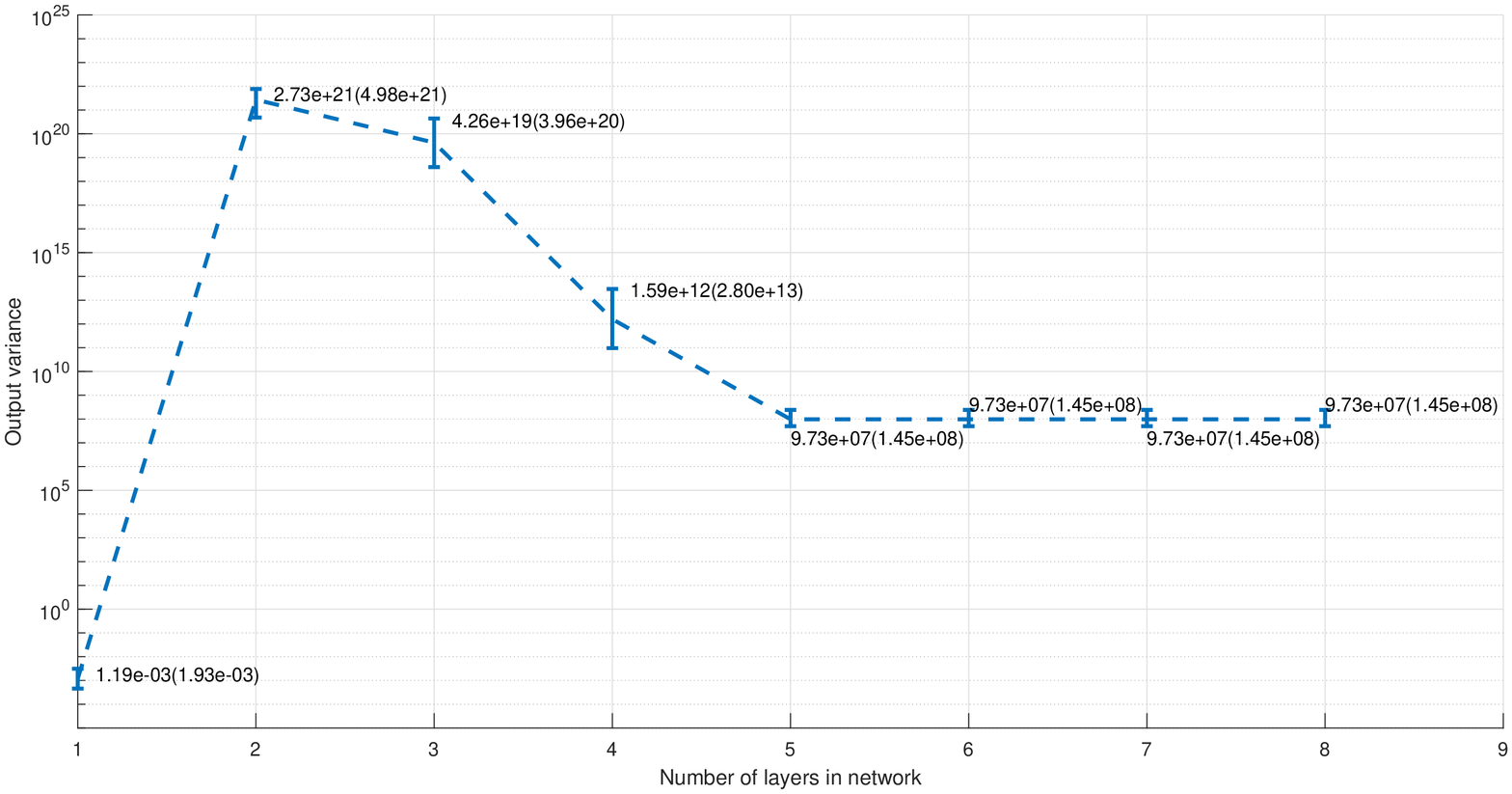}
        \caption{The average output variance of 1000 Monte Carlo simulations. The mean(std)
            values of $({\bf h}_{k}(\textbf{\textsf{x}}_0)\bm{\epsilon})^2$, $k=1,...,8$ are
            indicated for each error bar. Due to
            taking of logarithm on negative values, the lower portion of the error
            bar is a mere replica of the upper portion for visualization purpose.}
        \label{fig_output_variance}
    \end{center}
\end{figure}

\noindent{\bf Remark 3: } The analysis of output variance shows that the variance of
estimation does not diverge for deep networks when the hidden layers are estimated using
the data image (i.e., the pseudoinverse of the data matrix). This result is congruence
with the observation of good generalization obtained in deep networks in the literature.
Numerically, the result in Proposition~\ref{thm_output_variance} can be validated by a
Monte Carlo simulation as follows. To keep the outputs of deeper networks within the
visible changing range, an activation function of the exponential form (i.e.,
$e^{0.0001{\bf X}{\bf W}}$) has been adopted for each layer in this study. The weights
are computed deterministically according to \eqref{eqn_C3R1} through \eqref{eqn_C3R4}.
The elements of the input matrix ${\bf X}\in\RealR^{100\times 10}$ (100 samples of 10
dimensions) are generated randomly within $[-5,+5]$. The vector
$\textbf{\textsf{x}}_0\in\RealR^{100\times 1}$ in \eqref{eqn_h_gen} is also generated
randomly within the same range. The learning target vector ${\bf y}\in\{-1,+1\}$ contains
100 samples with half of the total samples for each category. The noise vector
$\bm{\epsilon}\in\RealR^{100\times 1}$ in \eqref{eqn_h_gen} is taken to be of one-tenth
magnitude of the input elements' range. A 1000 Monte Carlo simulations have been
conducted for the random setting. Fig.~\ref{fig_output_variance} shows the average plot
of $({\bf h}_k(\textbf{\textsf{x}}_0)\bm{\epsilon})^2$ for the 1000 random trials over
$k=1,2,\cdots,8$ which correspond to the number of layers of each network. By excluding
the linear network without the exponential activation at $k=1$ (i.e., $({\bf
    h}_1(\bm{x}_0)\bm{\epsilon})^2$ where $\bm{x}_0\in\RealR^{10\times 1}$ which has a
different dimension from that of $\textbf{\textsf{x}}_0\in\RealR^{100\times 1}$), it can
be seen that the value of the output variance $({\bf
    h}_k(\textbf{\textsf{x}}_0)\bm{\epsilon})^2$ does not diverge for networks with more than
two layers. This result may contribute to explaining the non-overfitting behavior of deep
networks.\Box

\section{Case Studies} \label{sec_synthetic}

In this section, two sets of synthetic data with known function generators are adopted to
observe the learning and representation properties of the proposed network. The first
data set studies the fitting ability in terms of training data representation with
respect to the scaling of the initial weights. The second data set studies the decision
boundary of a learned network as well as the number of nodes needed in the hidden layer
for representation.

Without loss of generality, a modified \texttt{softplus} function $f(x)=\log(0.8+e^{x})$
and its inverse given by $\g(x)=\log(e^{x}-0.8)$ are adopted in the learning algorithm in
Corollary~\ref{cor_soln1} (abbreviated as \texttt{ANNnet}) for all the numerical studies
and experiments. The computing platform is an Intel Core i7-6500U CPU at 2.59GHz with 8GB
of RAM.

\subsection{A regression problem}

In this example, we examine the effect of initial weight magnitudes (based on a scaling
factor $c$ which is multiplied to the initial weight matrices elementwise) and the effect
of network layers using a single dimensional regression function. A total of 8 training
target samples $y_i$, $i=1,2,\cdots,8$ are generated based on the function
$y=\sin(2x)/(2x)$ using $x_i\in\{1,2,\cdots,8\}$. Apart from this set of training
samples, another 10 sets (each set containing 8 samples) of training target samples are
generated by adding 20\% of noisy variations with respect to the training target range.
Our purpose is to observe the mapping capability of the network to learn these 11 curves
with different curvatures using different $c$ settings and the number of network layers.
In order to observe the fine fitting points of the network, for each curve, another set
of output samples is generated using a higher resolution input
$x_j\in\{0.90,0.91,\cdots,8.09,8.10\}$. This test set or observation set contains 721
samples.

Fig.~\ref{fig_Regression_ex}(a) and (b) respectively show the fitting results of a
2-layer \texttt{ANNnet} (of 8-1 structure) at $c=1$ and $c=0.1$. The $c$ value is a
scaling factor for each of the ${\bf R}_2$ to ${\bf R}_n$ matrices in
\eqref{eqn_R1}-\eqref{eqn_R4}. The result at $c=1$ (Fig.~\ref{fig_Regression_ex}(a) and
Table~\ref{table_effects}) shows a perfect fit for all the training data points but with
a high testing (prediction) Sum of Squared Errors (SSE) which indicate the large
difference between the output of the test samples and the underneath target (blue) curve.
The fitting result shows a `smoother' (with lower fluctuations) curve at $c=0.1$ than
that at $c=1$. In terms of the deviation of the predicted output from the target
function, the SSE for the case of $c=0.1$ shows a lower value than that of $c=1$ (see the
test SSE values in Table~\ref{table_effects}). These results demonstrate an over-fitting
case for Fig.~\ref{fig_Regression_ex}(a) and an under-fitting case for
Fig.~\ref{fig_Regression_ex}(b) in terms of the training data where the SSE of training
data is compromised.

The results for a 5-layer \texttt{ANNnet} (of 1-1-1-8-1 structure) at $c=1$ and at
$c=0.1$ show a similar trend of having a `smoother' fit in the lower $c$ value. However,
this smoother fit with lower fluctuations does not compromise the training SSE values
while maintaining a low SSE for the test data (see Table~\ref{table_effects}). This is
different from that of the 2-layer case where the smoother fit compromises in fitting
every data points. This fitting behavior can be observed from the almost zero training
SSE values for both $c$ value settings in the 5-layer network. These results show a
better generalization capability for the 5-layer network than the 2-layer network.

\begin{figure}[hhh]
    \begin{center}
        \begin{tabular}{cc}
            \epsfxsize=6cm
            \epsffile[66     9   871   441]{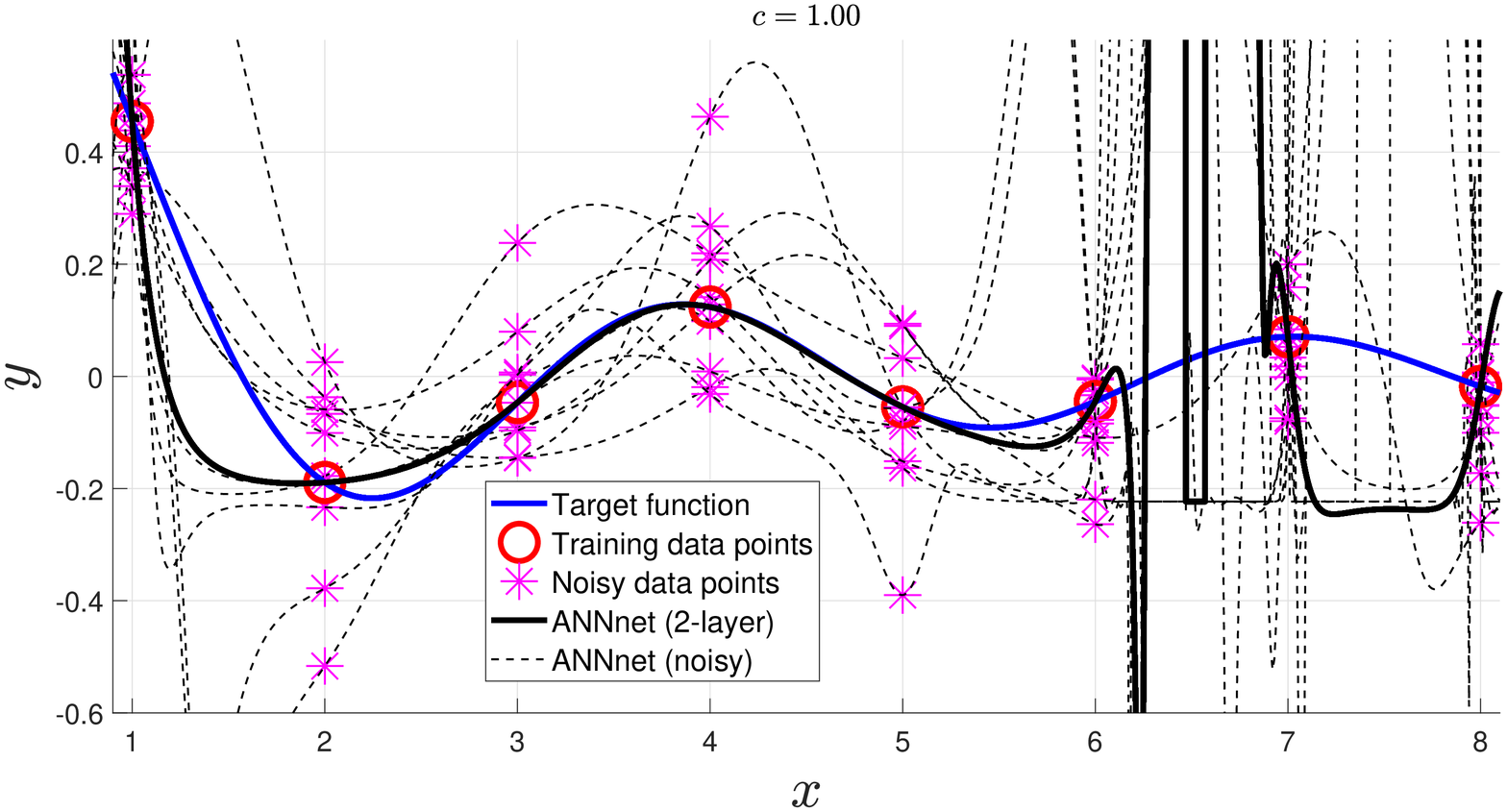}
            & \epsfxsize=6cm
            \hspace{0cm}
            \epsffile[66     9   871   441]{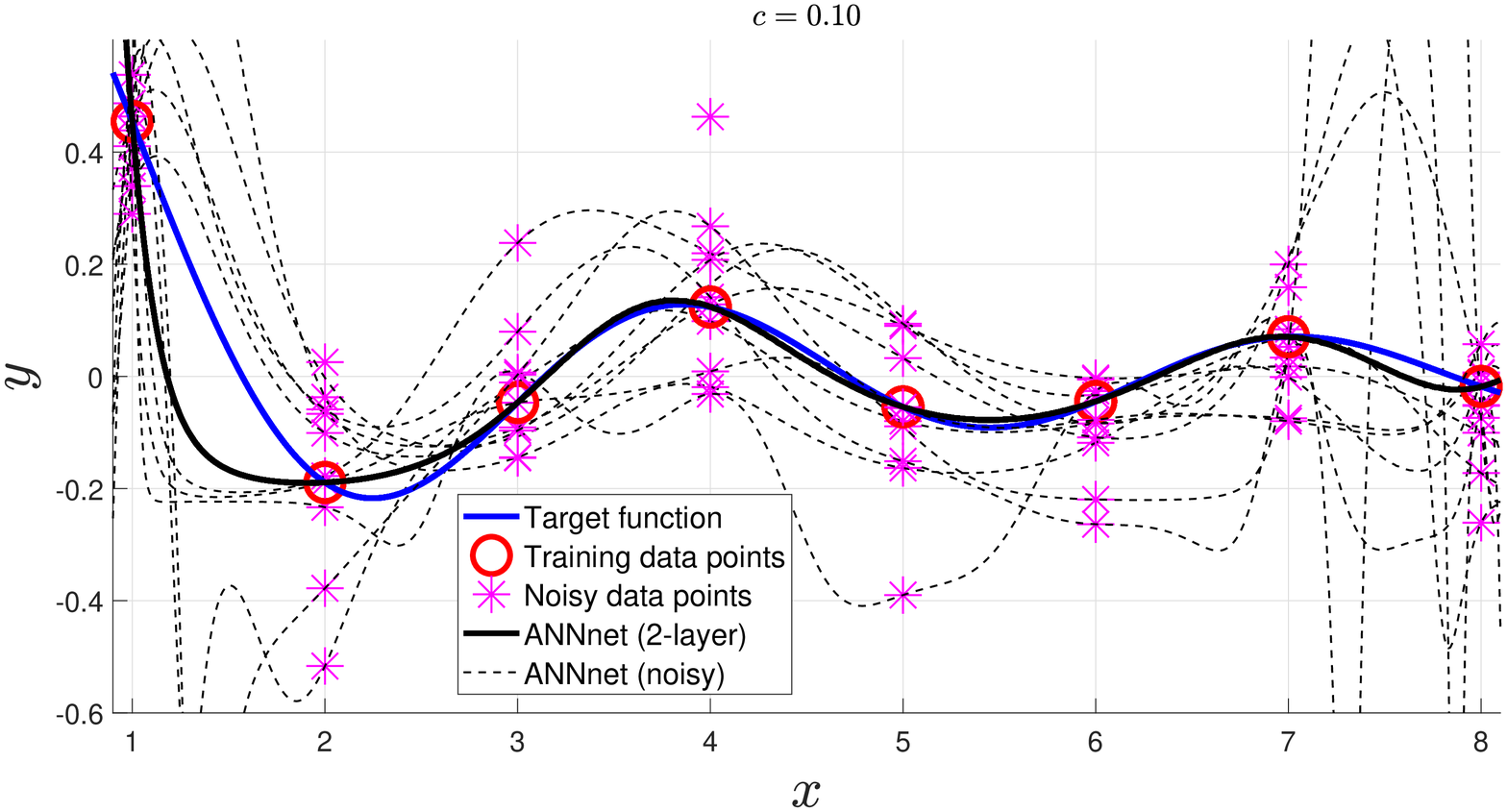}
            \\ \hspace{1cm} (a) & \hspace{1cm} (b) \\*[5mm]
            \epsfxsize=6cm
            \epsffile[66     9   871   441]{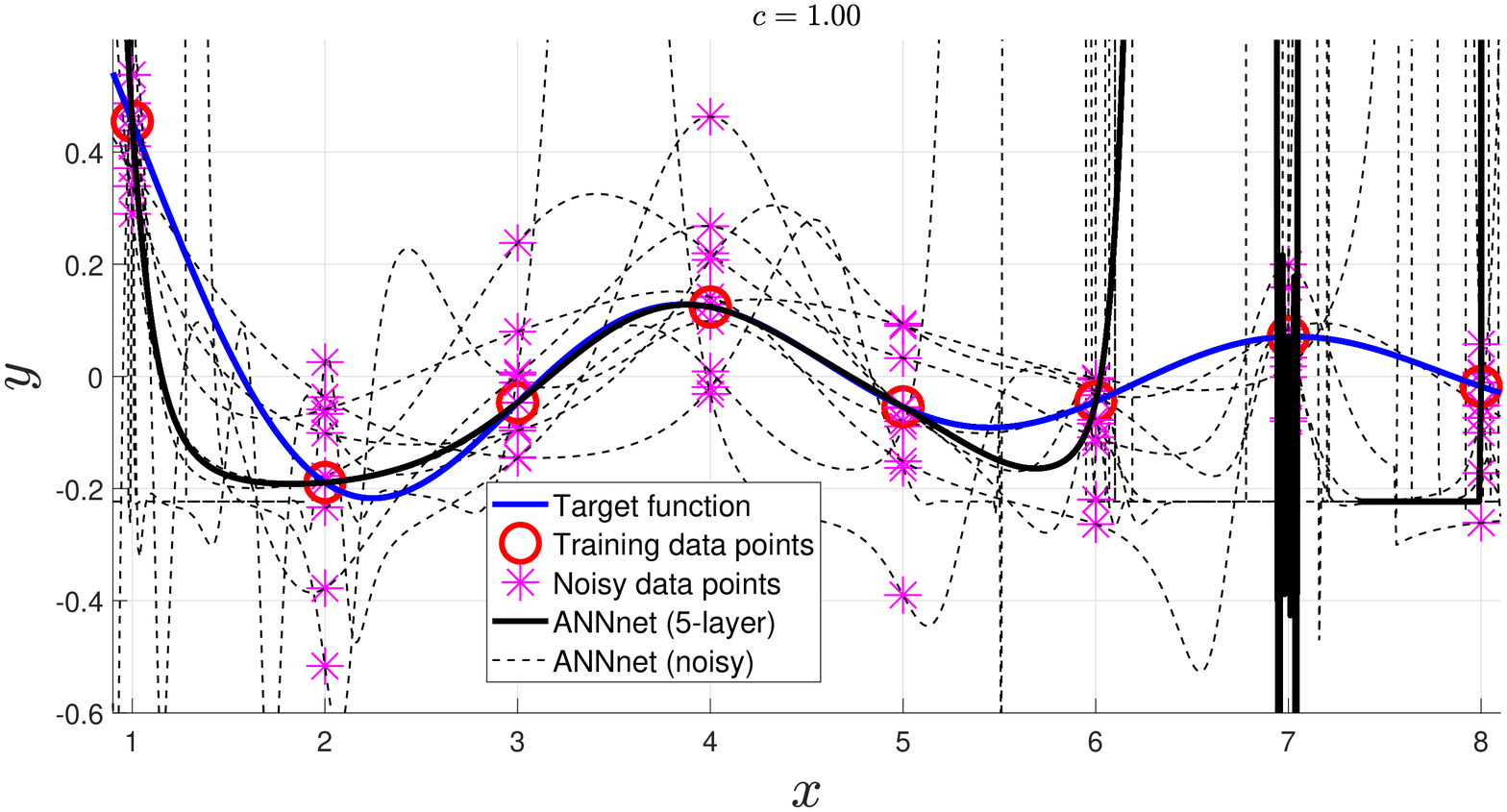}
            & \epsfxsize=6cm
            \hspace{0cm}
            \epsffile[66     9   871   441]{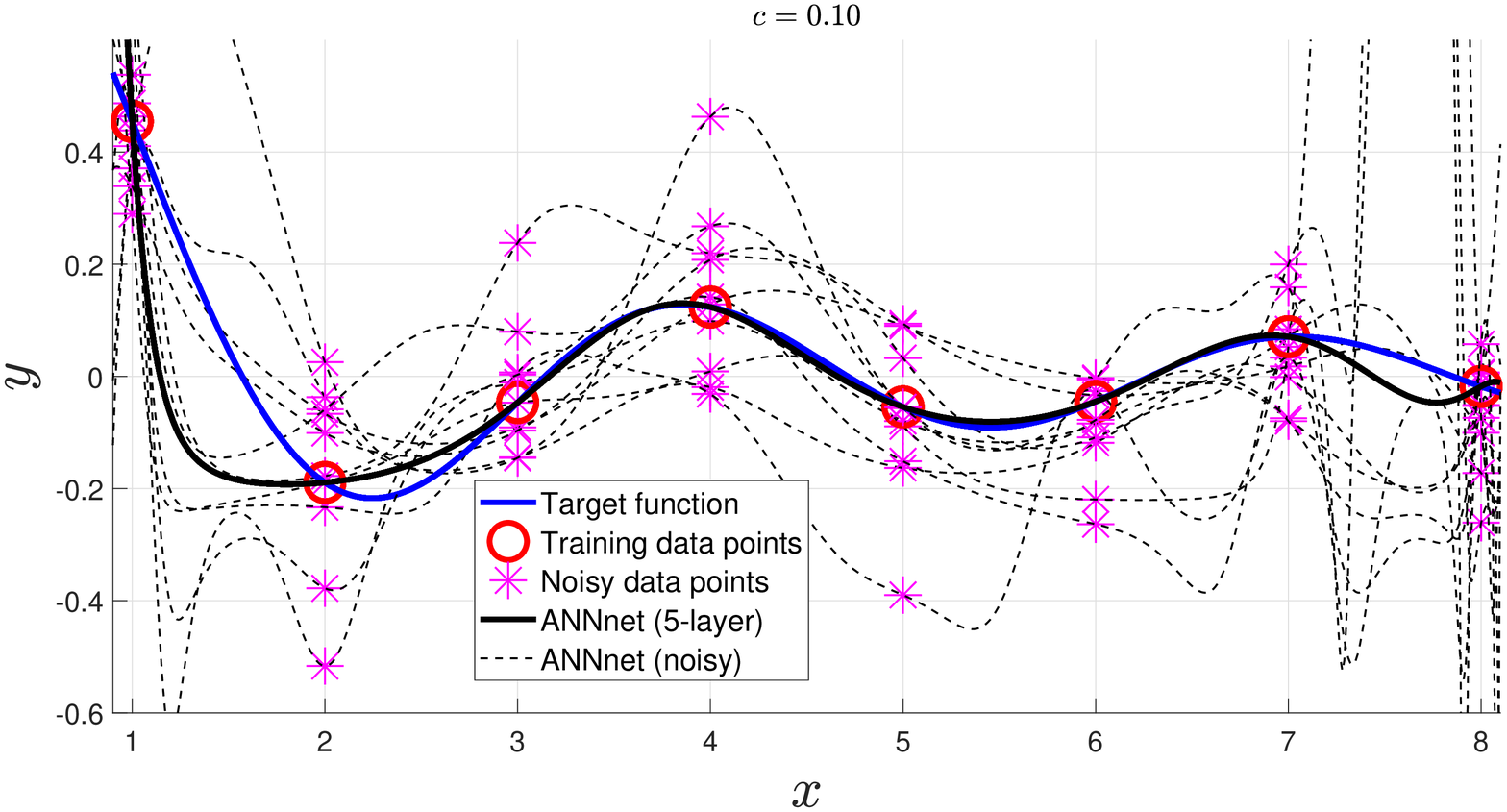}
            \\ \hspace{1cm} (c) & \hspace{1cm} (d) \\*[5mm]
        \end{tabular}
        \caption{Learned outputs of the \texttt{ANNnet} at different scaling $c$ settings for 2-layer
            and 5-layer networks.}
        \label{fig_Regression_ex}
    \end{center}
\end{figure}

\begin{table}[hhh]{\normalsize
        \begin{center}
            \caption{Sum of Squared Errors (training/test)} \label{table_effects}
            \begin{tabular}{|l|c|c|} \hline
                Methods                 &  $c=1$,   $s=0$                                 & $c=0.1$, $s=0$ \\ \hline
                2-layer \texttt{ANNnet} &  1.8421$\times 10^{-19}$/1.3512$\times 10^{5}$  & 1.1500$\times 10^{-2}$/1.1143$\times 10^{2}$ \\
                5-layer \texttt{ANNnet} &  9.0015$\times 10^{-13}$/1.2507$\times 10^{6}$  & 7.6328$\times 10^{-14}$/1.2426$\times 10^{2}$  \\
                \hline
            \end{tabular}\\*[2mm]
    \end{center} }
\end{table}

\subsection{The Spiral problem}

The spiral problem is well known in studying the mapping capability of neural networks.
It has been shown to be intrinsically connected to the inside-outside relations
\cite{ChenKe3}. In this example, a 6-arm spiral pattern has been generated with each arm
being perturbed by random noise (i.e., \texttt{rotation\_angle +
0.3}$\times$\texttt{rand}). Each arm contains a total of 500 samples with half being used
for training and the remaining half for test. In other words, the training set and the
test set each contains 1500 samples for the 6 arms ($1500=250\times 6$). A 4-layer
network of 150-250-150-6 structure has been used to learn the data using the training
set. The classification accuracy performance is measured by counting the number of test
samples which fall out of the decision boundary. By using the \texttt{softplus8}
activation function with $c=0.5$ for scaling the initial weights, an error free testing
result is achieved. Fig.~\ref{fig_Spiral_4layer} shows the learned decision boundary with
the test samples.

\begin{figure}[hhh]
    \begin{center}
        \epsfxsize=12cm
        \epsffile[62    29   524   379]{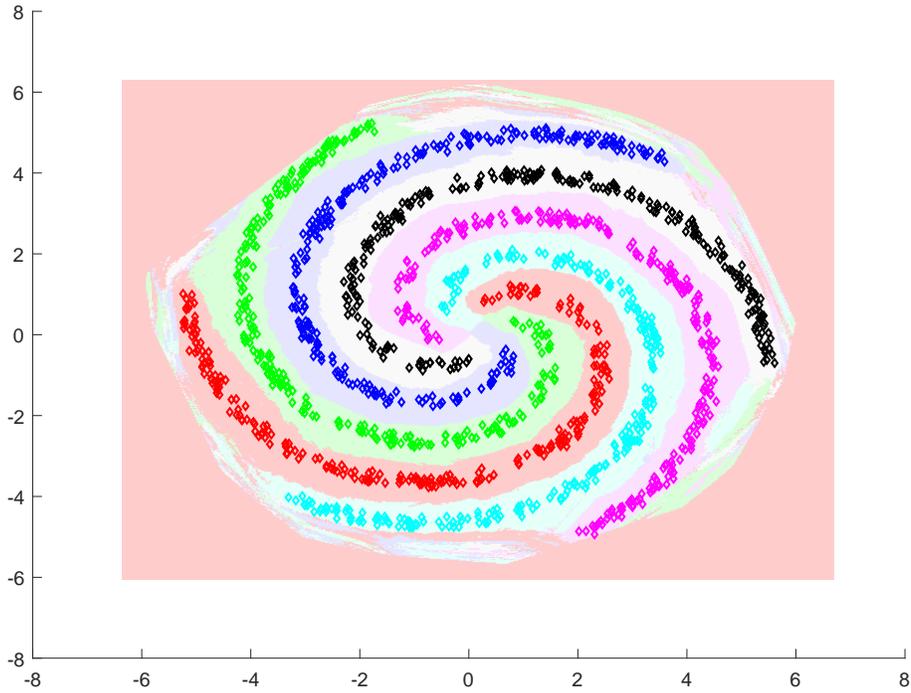}
        \caption{Testing data points and the decision boundary of a trained \texttt{ANNnet}.}
        \label{fig_Spiral_4layer}
    \end{center}
\end{figure}

Fig.~\ref{fig_Spiral_4layer_SSE} shows the Sum of Squared Errors (SSE) of training
results for 10 trials using different random seeds for the weights ${\bf R}_2,\cdots,{\bf
    R}_n$ initialization. The adopted network structure is similar to that of the above
except for the number of hidden nodes $h_{3}$ which is varied from 1250 to 1600 at an
interval of 50. In a nutshell, a network structure of 150-250-$h_{3}$-6 has been adopted
with $h_{3}$ varies within $\{1250,1300,\cdots,1550,1600\}$ in order to observe its
effect on network representation. The purpose of choosing $h_{3}$ in this range is to
cover the training sample size (1500) for observing the accuracy of fit. The results in
Fig.~\ref{fig_Spiral_4layer_SSE} show zero SSE for $h_{3}\geq 1500$ which reflect the
representation capability of the network for the 1500 training samples.

\begin{figure}[hhh]
    \begin{center}
        \epsfxsize=8cm
        \epsffile[22     4   391   295]{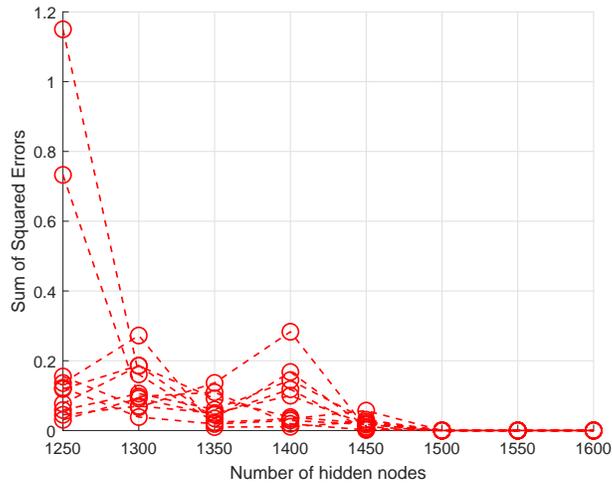}
        \caption{Variation of the training Sum of Squared Errors (SSE) over the number of hidden nodes at layer $h_3$ for 10 trails based on different random seeds.}
        \label{fig_Spiral_4layer_SSE}
    \end{center}
\end{figure}

Fig.~\ref{fig_6Spiral_h1h2}(a) and (b) show respectively the SSEs for variation of $h_2$
and $h_1$ nodes while fixing the other two layers at 150 nodes. These results show a
variation of the minimum number of hidden nodes needed for training data representation
(in order to arrive at an almost zero training SSE) towards the input layer.

\begin{figure}[hhh]
    \begin{center}
        \begin{tabular}{cc}
            \epsfxsize=6cm
            \epsffile[30     4   391   305]{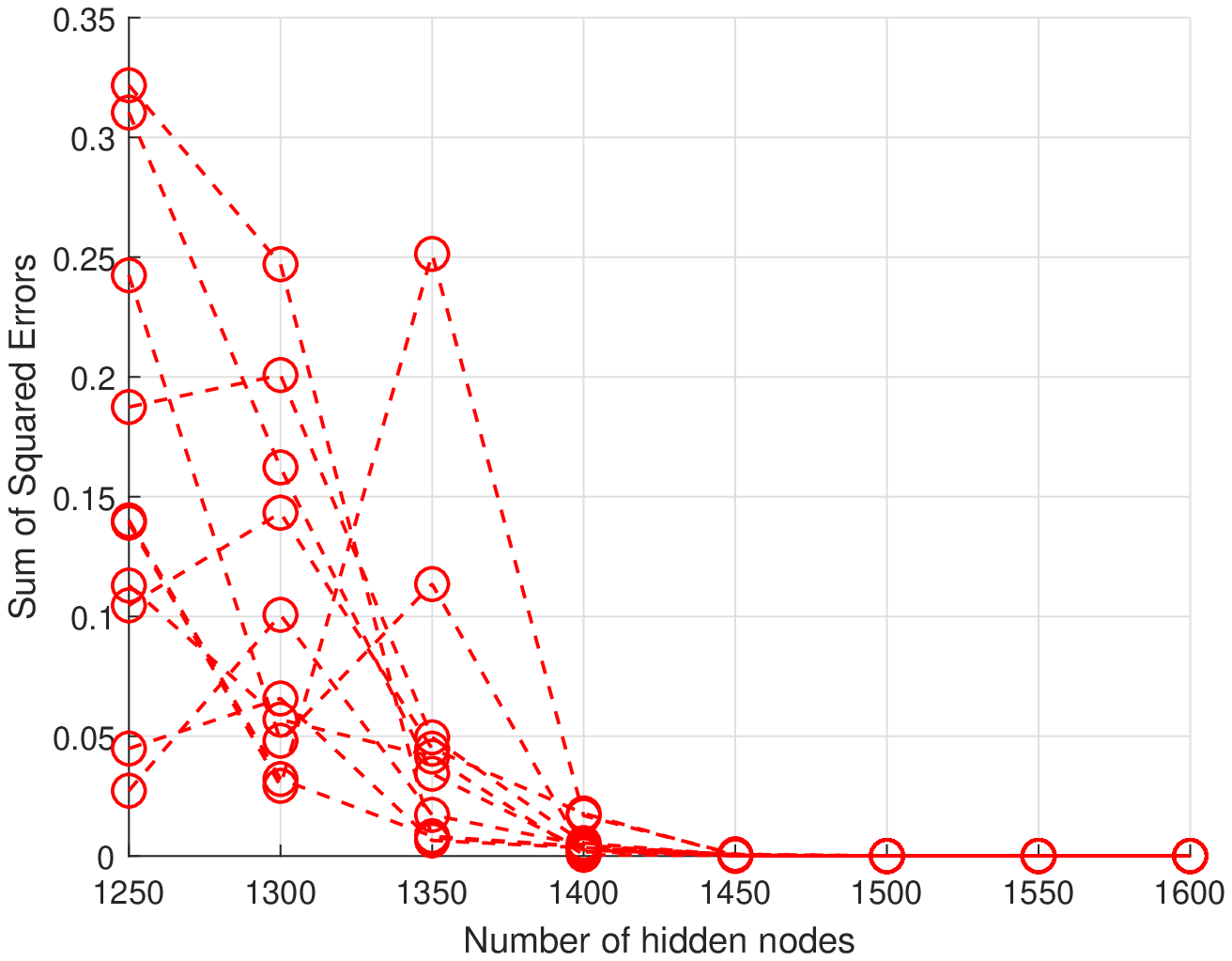}
            & \epsfxsize=6cm
            \hspace{0cm}
            \epsffile[17     4   391   295]{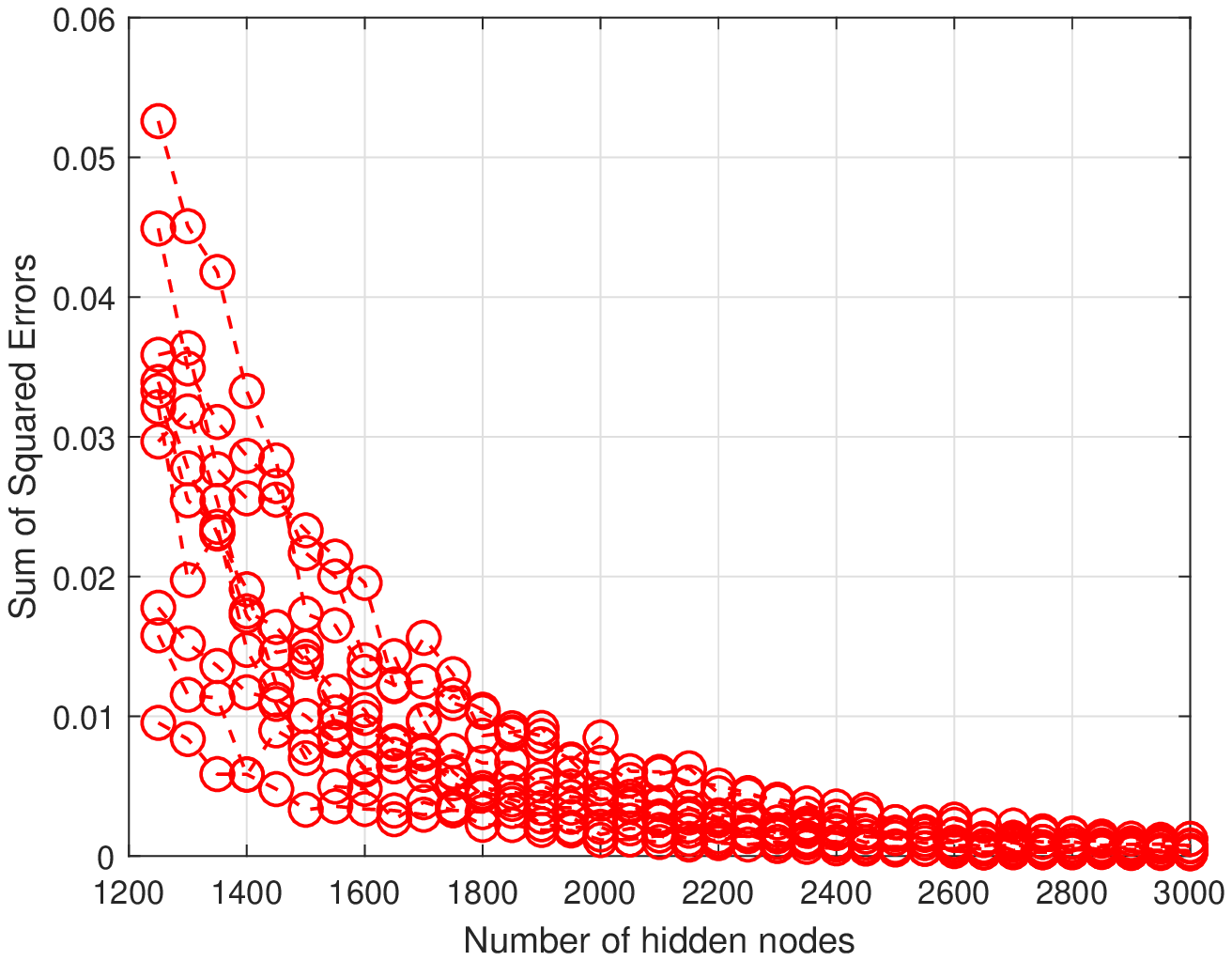}
            \\  (a) Number of $h_2$ hidden nodes &  (b) Number of $h_1$ hidden nodes\\*[5mm]
        \end{tabular}
        \caption{Variation of the training Sum of Squared Errors (SSE) over the number of hidden nodes ((a) at layer $h_2$ and (b) at layer $h_1$) for 10 trails based on different random seeds.}
        \label{fig_6Spiral_h1h2}
    \end{center}
\end{figure}

\section{Experiments} \label{sec_expts}

In this section, the proposed learning algorithm is evaluated using real-world data sets
taken from the UCI Machine Learning Repository \cite{UCI1b}.  Our goals are (i) to
observe whether the gradient-free learning is numerically feasible for real-world data of
small to medium size on a computer with 8 GB of RAM and without GPU? If feasible, how is
the accuracy comparing with that of the conventional network learning? (ii) the
computational CPU performance, (iii) Finally, the effects of network depth and sparseness
are evaluated using 2-, 3- and 5-layer networks on the relatively humble computer
platform.

\subsection{UCI Data Sets}

The UCI data sets are selected according to \cite{Toh39} and are summarized in
Table~\ref{table_summary_attributes} with their pattern attributes. The experimental
goals (i) and (iii) are evaluated by observing the prediction accuracy of the algorithm.
The accuracy is defined as the percentage of samples being classified correctly. The
experimental goal (ii) is observed by recording the training CPU processing time of the
compared algorithms on an Intel Core i7-6500U CPU at 2.59GHz with 8GB of RAM memory. The
training CPU time is measured using the Matlab's function \texttt{cputime} which
corresponds to the total computational times from each of the 2 cores in the Processor.
In other words, the physical time experienced is about 1/2 of this clocked
\texttt{cputime}.

\begin{table}
    \begin{center}\vspace{-6mm}
        \caption{Summary of UCI \cite{UCI1b} data sets and chosen hidden layer sizes ($h$) based
            on cross-validation within the training set} \label{table_summary_attributes}{\tiny
            \begin{tabular}{|cl|cccc|cc|c|c|} \hline
                &                     & (i)       & (ii)      & (iii)     & (iv)    &
                \multicolumn{2}{c}{2-layer} \vline & 3-layer    & 5-layer \\
                &                     &           &           &           &         &
                ( $h$ )           & ( $h$ ) & ( $h$ ) & ( $h$ ) \\
                & Database name       & {\#}cases & {\#}feat  & {\#}class & {\#}miss & FFnet & \texttt{ANNnet} & \texttt{ANNnet} & \texttt{ANNnet} \\ \hline
                1.  & Shuttle-l-control   & 279(15)   & 6         & 2         & no       &  3   &  3 &   5 &   5  \\
                2.  & BUPA-liver-disorder & 345       & 6         & 2         & no       &  2   & 500&   3 &   2  \\
                3.  & Monks-1             & 124(432)  & 6         & 2         & no       &  5   &  1 &   2 &   2  \\
                4.  & Monks-2             & 169(432)  & 6         & 2         & no       &  10  & 200& 500 &  80  \\
                5.  & Monks-3             & 122(432)  & 6         & 2         & no       &  5   &  1 &   3 &   5 \\
                6.  & Pima-diabetes       & 768       & 8         & 2         & no       &  3   &  2 &   3 &  10 \\
                7.  & Tic-tac-toe         & 958       & 9         & 2         & no       &  20  & 30 & 500 &  10 \\
                8.  & Breast-cancer-Wiscn & 683(699)  & 9(10)     & 2         & 16       &  10  &  3 &   3 &   3 \\
                9.  & StatLog-heart       & 270       & 13        & 2         & no       &  2   & 20 &   3 &   3 \\
                10. & Credit-app          & 653(690)  & 15        & 2         & 37       &  1   &  3 &   1 &   5 \\
                11. & Votes               & 435       & 16        & 2         & yes      &  10  &  2 &   3 &   3 \\
                12. & Mushroom            & 5644(8124)& 22        & 2         & attr{\#}11 & 5  & 80 &  50 &  10 \\
                13. & Wdbc                & 569       & 30        & 2         & no       &  2   & 10 &   2 &   5 \\
                14. & Wpbc                & 194(198)  & 33        & 2         & 4        &  1   &  5 & 500 &  80 \\
                15. & Ionosphere          & 351       & 34        & 2         & no       &  10  &  1 &   2 &   2 \\
                16. & Sonar               & 208       & 60        & 2         & no       &  30  &  5 &   3 &   3 \\
                \hline
                17. & Iris                & 150       & 4         & 3         & no       &  20  & 10 &  10 &   5 \\
                18. & Balance-scale       & 625       & 4         & 3         & no       &  20  & 50 &  20 &   5 \\
                19. & Teaching-assistant  & 151       & 5         & 3         & no       &  50  & 500&  80 &  80 \\
                20. & New-thyroid         & 215       & 5         & 3         & no       &  5   & 20 &  10 &   5 \\
                21. & Abalone             & 4177      & 8         & 3(29)     & no       &  50  & 30 &  20 &  10 \\
                22. & Contraceptive-methd & 1473      & 9         & 3         & no       &  20  & 50 &  20 &   3 \\
                23. & Boston-housing      & 506       & 12(13)    & 3(cont)   & no       &  50  & 50 &  10 &   5 \\
                24. & Wine                & 178       & 13        & 3         & no       &  50  & 30 &  10 &   5 \\
                25. & Attitude-smoking$^+$& 2855      & 13        & 3         & no       &  1   &  1 &   1 &   1 \\
                26. & Waveform$^+$        & 3600      & 21        & 3         & no       &  20  &  5 &  20 &   5 \\
                27. & Thyroid$^+$         & 7200      & 21        & 3         & no       &   3  & 80 &  20 &  30 \\
                28. & StatLog-DNA$^+$     & 3186      & 60        & 3         & no       &  10  & 20 &  10 &  20 \\
                \hline
                29. & Car                 & 2782      & 6         & 4         & no       &  200 & 500& 200 &  80 \\
                30. & StatLog-vehicle     & 846       & 18        & 4         & no       &  10  & 50 &  20 &  10 \\
                31. & Soybean-small       & 47        & 35        & 4         & no       &   1  &  3 &   3 &   3 \\
                32. & Nursery             & 12960     & 8         & 4(5)      & no       & 100  & 500&  80 & 100 \\
                33. & StatLog-satimage$^+$& 6435      & 36        & 6         & no       &   5  & 500& 200 &  20 \\
                34. & Glass               & 214       & 9(10)     & 6         & no       &  80  & 10 &  20 &   5 \\
                35. & Zoo                 & 101       & 17(18)    & 7         & no       &  20  & 20 &  50 & 100 \\
                36. & StatLog-image-seg   & 2310      & 19        & 7         & no       & 100  & 500& 200 &  -- \\
                37. & Ecoli               & 336       & 7         & 8         & no       &  20  & 50 &  20 &   5 \\
                38. & LED-display$^+$     & 6000      & 7         & 10        & no       & 100  & 50 &  20 & 100 \\
                39. & Yeast               & 1484      & 8(9)      & 10        & no       & 100  & 100&  30 &  10 \\
                40. & Pendigit            & 10992     & 16        & 10        & no       & 200  & 500& 500 &  50 \\
                41. & Optdigit            & 5620      & 64        & 10        & no       & 200  & 500& 500 &  30 \\
                42. & Letter              & 20000     & 16        & 26        & no       & --   & 500& 500 & 200 \\
                \hline
        \end{tabular} } {\tiny
            \begin{tabular}{llp{12.5cm}}
                (i-iv) &: & (i) Total number of instances, i.e. examples, data points, observations
                (given number of instances). Note: the number of instances used is larger than the given
                number of instances when we expand those ``don't care'' kind of attributes in some data
                sets; (ii) Number of features used, i.e. dimensions, attributes (total number of features
                given); (iii) Number of classes (assuming a discrete class variable);
                (iv)  Missing features; \\
                $+$ &:   & Accuracy measured from the given training and test set instead of 10-fold
                validation (for large data cases with test set containing at least 1,000 samples); \\
                FFnet &: & The \texttt{feedforwardnet} from the Matlab's toolbox using default settings; \\
                $h$  &: & The number of hidden nodes for 2layer, 3layer, and 5layer networks are set
                as $h$-$q$, $2h$-$h$-$q$, and $8h$-$4h$-$2h$-$h$-$q$ respectively; \\
                Note &:  & Data from the Attitudes Towards Smoking Legislation Survey - Metropolitan
                Toronto 1988, which was funded by NHRDP (Health and Welfare Canada), were collected by
                the Institute for Social Research at York University for Dr. Linda Pederson and Dr.
                Shelley Bull.
        \end{tabular} }
    \end{center}
\end{table}

\subsection*{(i) Feasibility and prediction accuracy of two-layer networks}

In this experiment, the prediction accuracy of \texttt{ANNnet} is recorded and compared
with that of the \texttt{feedforwardnet} (abbreviated as \texttt{FFnet}) from the
Matlab's toolbox \cite{Matlab} based on the fully connected network structure of
two-layers. The activation function adopted for \texttt{ANNnet} is \texttt{softplus} with
random weights initialization, whereas \texttt{FFnet} adopted the default
`\texttt{tansig}' activation function (since \texttt{softplus} does not converge well
enough in this case) with the `\texttt{trainlm}' learning search. The test performance is
evaluated based on 10 trails of 10-fold stratified cross-validation tests for each of the
data set. The selection of the number of hidden nodes for each network is based on
another 10-fold cross-validation within the training set. The search for the number of
hidden nodes $h$ is conducted within the set $\{$1, 2, 3, 5, 10, 20, 30, 50, 80, 100,
200, 500$\}$. Table~\ref{table_summary_attributes} summarizes the chosen hidden node
sizes for the experimented networks for each data set. Apparently, there seems to have no
strong correlation between the choice of hidden node sizes for the two networks of
two-layer.

Fig.~\ref{fig_compare1} shows the average prediction accuracy recorded from the 10 trails
of stratified 10-fold cross-validation tests. The plot verifies (i) the feasibility of
gradient-free computation for real-world data with none of the results running into
computational ill-conditioning. Moreover, the results show a comparable average
prediction accuracy for \texttt{ANNnet} (with grand average accuracy of 82.03\% on the
first 41 data sets) and \texttt{FFnet} (with grand average accuracy of 82.36\% on the
first 41 data sets) for most data sets. The result for the 42nd data set for
\texttt{FFnet} was not available due to insufficient memory in our computing platform. As
indicated by the shaded regions, these results show a significantly larger fluctuation of
prediction accuracies for \texttt{FFnet} (where its results vary within the green region)
than that of \texttt{ANNnet} (red region) over the 10 trials of 10-fold cross-validation
tests.

\begin{figure}[hhh]
    \begin{center}
        \epsfxsize=12cm
        \epsfysize=6cm
        \epsffile[93    19   870   428]{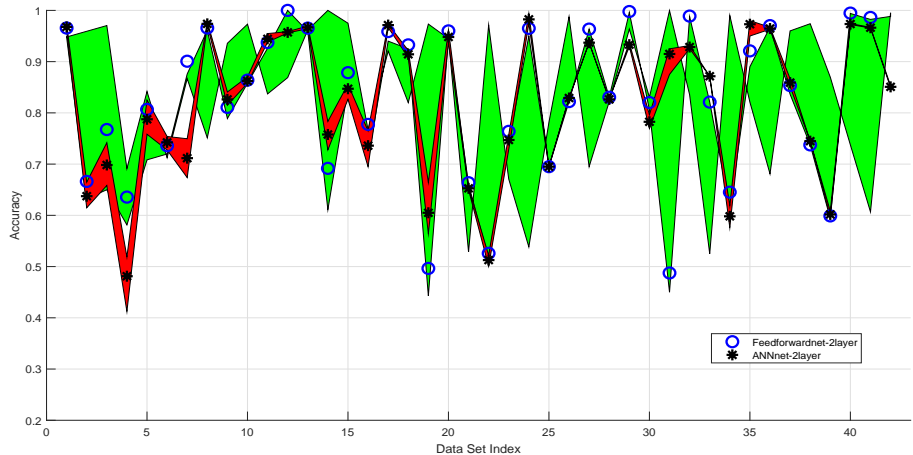}
        \caption{Prediction accuracy based on the average of 10 trials of 10-fold cross-validations.
            The shaded region depicts the maximum and minimum accuracy bounds and the mean accuracies for
            \texttt{FFnet} (green) and \texttt{ANNnet} (red) are respectively indicated by circles and asterisks.}
        \label{fig_compare1}
    \end{center}
\end{figure}

\subsection*{(ii) The training CPU processing time}

Fig.~\ref{fig_compare_CPU1} shows the training CPU processing times (which are measured
using the \texttt{cputime} function that accounts for per core processing time) of
\texttt{FFnet} and \texttt{ANNnet} based on the same hidden node size for each data set.
These results show at least an order (10 times) of speed-up for \texttt{ANNnet} over
\texttt{FFnet} in terms of the training time. The maximum ratio of training time speed-up
(CPU(\texttt{FFnet})/CPU(\texttt{ANNnet})) is $3.47\times 10^7$. The main reason for the
speed up in training time is the gradient-free analytic training solution. Although the
training time can be much faster than the conventional search mechanism, it is noted that
the proposed method requires a relatively large size of RAM memory to compute the
pseudoinverse of the entire data matrix.

\begin{figure}[hhh]
    \begin{center}
        \epsfxsize=12cm
        \epsfysize=6cm
        \epsffile[93    19   870   428]{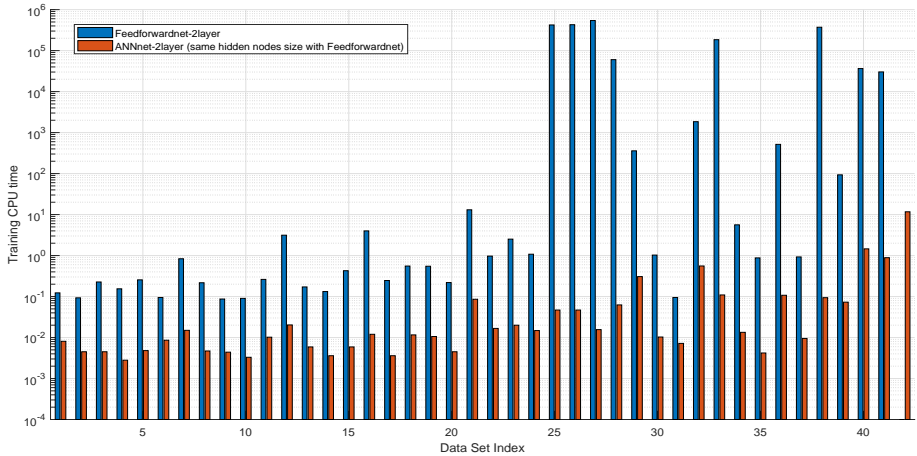}
        \caption{Comparing the training CPU time of \texttt{ANNnet} with that of \texttt{Feedforwardnet} adopting the
            same number of hidden nodes.}
        \label{fig_compare_CPU1}
    \end{center}
\end{figure}

\subsection*{(iii) The effect of deep layers and sparseness}

\noindent\textbf{Effect of deep layers}: Fig.~\ref{fig_compare_deeplayers} shows the
average accuracy of \texttt{ANNnet} with 2-layer, 3-layer and 5 layer structures plotted
respectively for each data set. The 2-layer network uses a structure of $h$-$q$ with
hidden layer size $h$ and output layer size $q$. The 3-layer network uses a $2h$-$h$-$q$
structure and the 5-layer network uses a $8h$-$4h$-$2h$-$h$-$q$ structure. The size of
$h$ is determined based on a cross-validation search within the training set for
$h\in\{$1, 2, 3, 5, 10, 20, 30, 50, 80, 100, 200, 500$\}$. The results in
Fig.~\ref{fig_compare_deeplayers} shows that the 5-layer \texttt{ANNnet} outperformed the
other two networks for many data sets. The grand average results for \texttt{ANNnet} of
2-layer, 3-layer and 5 layer are respectively 81.75\%, 81.28\% and 84.13\%. These results
appear to support good generalization for the network of 5 layers.

\begin{figure}[hhh]
    \begin{center}
        \epsfxsize=12cm
        \epsfysize=6cm
        \epsffile[93    19   870   428]{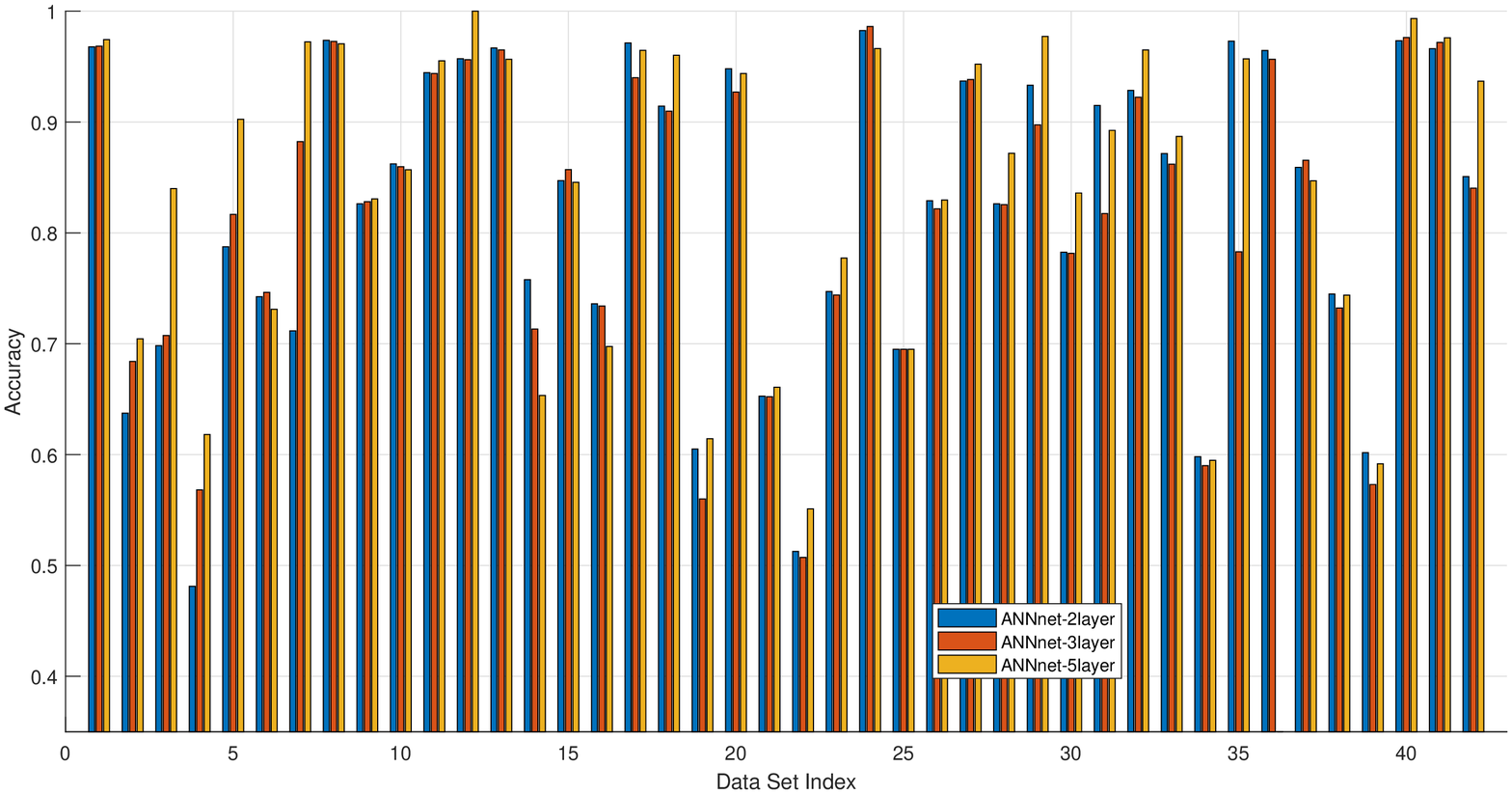}
        \caption{The effect of layers on accuracy performance.}
        \label{fig_compare_deeplayers}
    \end{center}
\end{figure}

\noindent\textbf{Effects of weight scaling and sparseness}: When the scaling factor, $c$,
for random initialization was set at 0.1, the grand average accuracy for \texttt{ANNnet}
was slightly improved (82.21\% on the first 41 data sets according to the results in (i))
for the two-layer network. As for the sparseness setting (receptive field $r$) the gross
average accuracies are observed to be 82.30\% and 83.02\% respectively for the
fully-connected and sparsely-connected (with receptive field of $r=3$ units in the first
layer) two-layer \texttt{ANNnet}s. In other words, the fully connected network and the
sparse network are with structures $h_1$-$h_2$ and $h_1^{r_3}$-$h_2$ respectively. For
the three-layer \texttt{ANNnet}, the observed gross average accuracies are 81.82\% and
81.93\% respectively for the fully-connected ($h_1$-$h_2$-$h_3$) and the
sparsely-connected ($h_1^{r_3}$-$h_2$-$h_3$) cases. These results (in
Fig.~\ref{fig_scaling_sparseness}) show a more significant impact of the weight scaling
than the receptive field (sparseness) on the generalization performance.

\begin{figure}[hhh]
    \begin{center}
        \begin{tabular}{cc}
            \epsfxsize=6.8cm
            \epsffile[46     8   544   329]{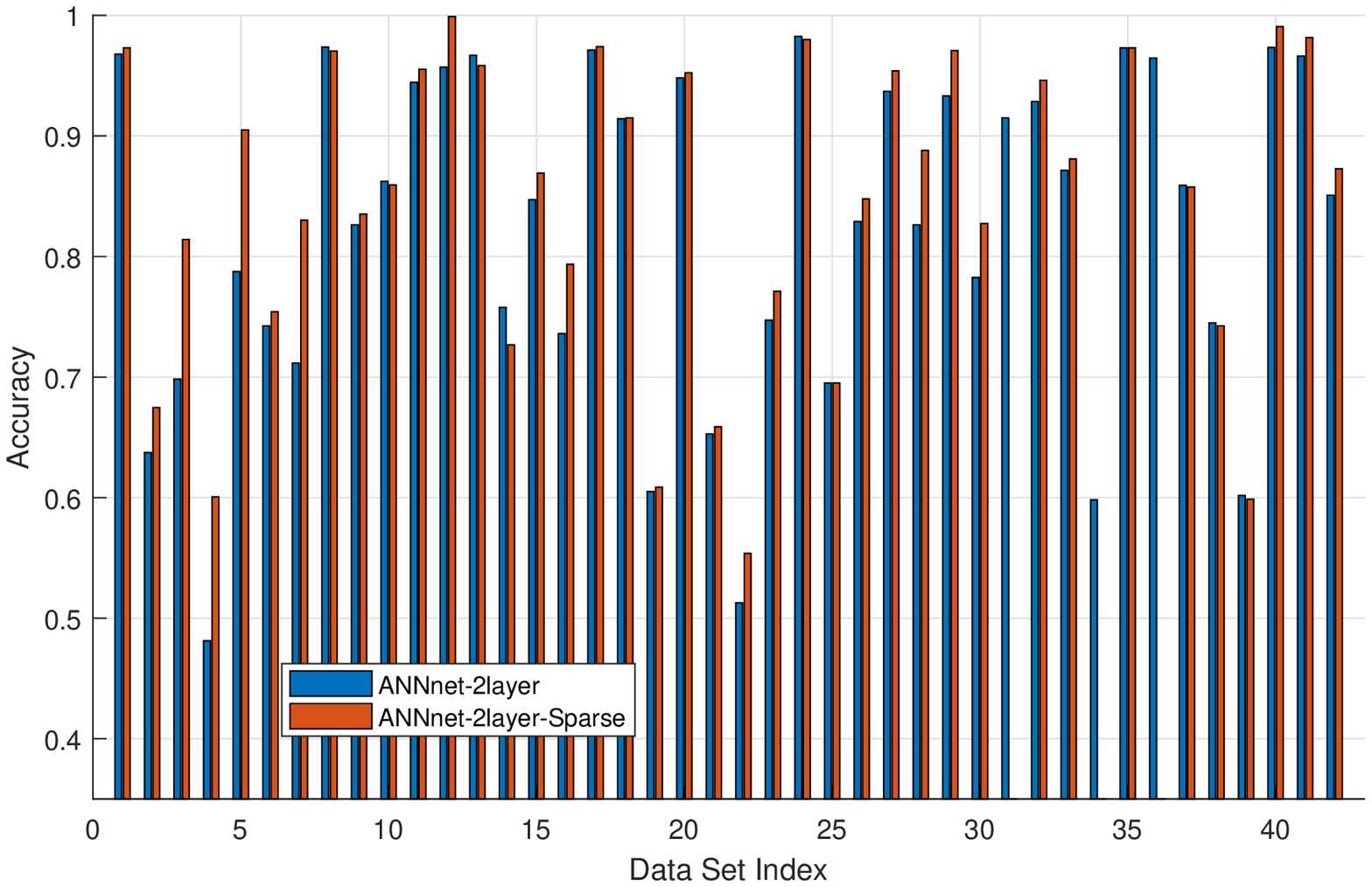}
            & \epsfxsize=6.8cm
            \hspace{0cm}
            \epsffile[46     8   544   329]{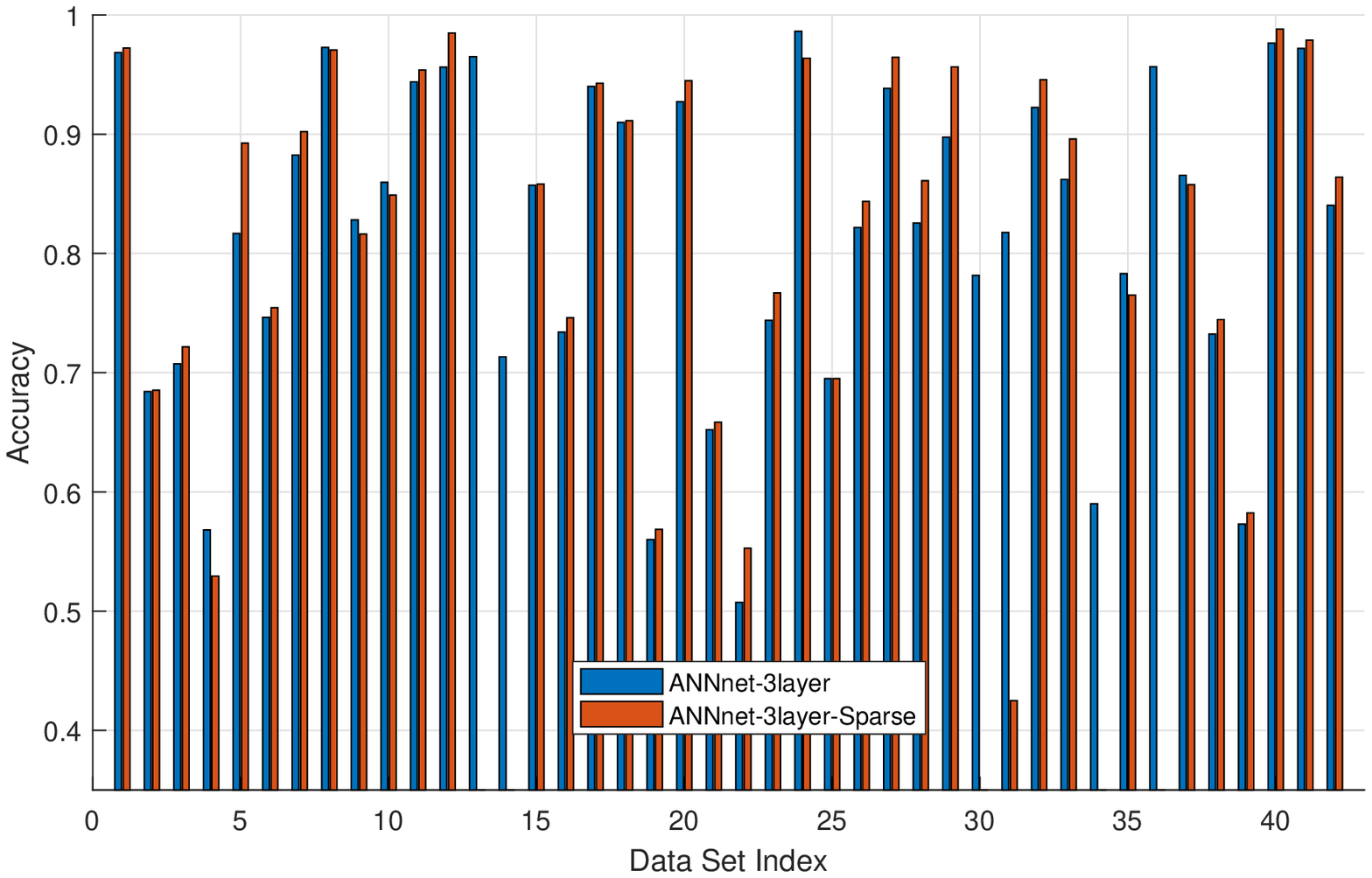}
            \\  (a) Sparse 2-layer \texttt{ANNnet} &  (b) Sparse 3-layer \texttt{ANNnet}\\*[5mm]
        \end{tabular}
        \caption{Comparing a non-sparse network (blue bars) with a sparse network (red bars):
            (a) 2-layer \texttt{ANNnet}, (b) 3-layer \texttt{ANNnet}.}
        \label{fig_scaling_sparseness}
    \end{center}
\end{figure}

\subsection{Summary of Results and Discussion}

\subsection*{Summary of Experimental Results}

In summary, the three goals of the experiments have been achieved as follows:
\begin{itemize}
    \item Based on an extensive numerical study utilizing 42 real-world data sets of small to medium
    sizes in terms of their dimension and sample sizes, the numerical feasibility of network
    learning based on the Moore-Penrose inverse is clearly verified. The structures of
    networks being studied included two-, three- and five-layers where few serious
    numerical stability issues are observed.

    \item The prediction accuracy of the proposed learning is observed to be comparable
    with that of the conventional gradient based learning. Attributed to the analytic
    learning formulation, the training processing time is observed to be at least 10
    times faster than that of the conventional learning. For some data sets, the
    maximum learning speed-up can go as high as $3.47\times 10^7$.

    \item Based on the study on networks of two-, three- and five-layers, the five-layer
    network shows more than 2\% better prediction accuracy than that of the other two networks
    in terms of the grand average accuracy over all data sets.

    \item The study shows that both the weights scaling and the sparseness of
    weights affect the generalization performance. This shall be a good topic for further
    research in the future.
\end{itemize}

\subsection*{Discussion and Future Works}

Capitalized on the classical learning theory on Moore-Penrose inverse, an analytical
learning formulation for multilayer neural networks has been proposed. Essentially, the
gradient-free learning is based on the minimization approximation through the
Moore-Penrose inverse projection via the column and row spaces. Based on the projection
manipulation, it turns out that the solutions to the layer weights are interdependent.
However, thanks to the enormous number of feasible solutions available in the network, an
analytical learning can be obtained by having the weights initialized layer by layer.
Though not limited to randomness, the feasibility of such a solution has been
demonstrated by randomly initialized weights in experimentation. Based on these results,
the following possibilities for future research are observed.

\begin{itemize}
    \item The formulation based on the layered full matrix network structure with
    elementwise operated activation functions does not depart from the conventional system of linear equations.
    Thanks to the inherent least squares approximation property of the Moore-Penrose inverse, such a formulation allows
    the network solution to be written in analytical form. This treatment in the linear
    system form not only makes the formulation analytic and transparent but also opens
    up vast research possibilities towards understanding of network learning and
    optimization.

    \item An immediate possibility for such research is regarding the effect of network
    depth towards prediction generalization. Attributed to the layered matrix structure
    that hinged upon the system of linear equations, the number of possible
    alternative solution to the system has been shown to be exponentially increased.
    This opens up the vast possibilities in initializing the network weights
    towards good generalization.

    \item Other possibilities include the structural construction particularly the sparse structure and the scale of the initial weights where they are found to influence the prediction generalization in the numerical experiments.

\end{itemize}

\section{Conclusion} \label{sec_conclusion}

Capitalized on the inherent property of least squares approximation offered by the
Moore-Penrose inverse, a gradient-free approach for solving the learning problem of
multilayer neural networks was proposed. The solution obtained from such an approach
showed that network learning boiled down to solving a set of inter-dependent weight
equations. By initializing the inner weight matrices either by random matrices or by the
data matrix, it turned out that the interdependency of weight equations can be decoupled
where an analytic learning can be achieved. Based on the analytic solution, the network
representation and generalization properties with respect to the layer depth were
subsequently analyzed. Our numerical experiments on a wide range of data sets of small to
medium sizes not only validated the numerical feasibility, but also demonstrated the
generalization capability of the proposed network solution.

\bibliographystyle{c:/ibb/tex1/IEEEtran2003}

\begin{thebibliography}{10}


\bibitem{LeCun3}
Y.~{LeCun} and Y.~Bengio, ``Convolutional networks for images, speech, and time
series,'' \emph{The handbook of brain theory and neural networks}, pp.
255--258, 1998.

\bibitem{AlexK1}
A.~Krizhevsky, I.~Sutskever, and G.~E. Hinton, ``Imagenet classification with
deep convolutional neural networks,'' in \emph{Proceedings of the 25th
    International Conference on Neural Information Processing Systems (NIPS
    2012)}, 2012, pp. 1097--1105.

\bibitem{Szegedy1}
C.~Szegedy, W.~Liu, Y.~Jia, P.~Sermanet, S.~Reed, D.~Anguelov, D.~Erhan,
V.~Vanhoucke, and A.~Rabinovich, ``Going deeper with convolutions,'' in
\emph{Proc. International Conference on Computer Vision and Pattern
    Recognition (CVPR)}, 2015, pp. 1--9.

\bibitem{Simonyan14verydeep}
K.~Simonyan and A.~Zisserman, ``Very deep convolutional networks for
large-scale image recognition,'' 2014. [Online]. Available:
{https://arxiv.org/abs/1409.1556}

\bibitem{HeKaiming1}
K.~He, X.~Zhang, S.~Ren, and J.~Sun, ``Deep residual learning for image
recognition,'' in \emph{Proc. International Conference on Computer Vision and
    Pattern Recognition (CVPR)}, 2016, pp. 770--778.

\bibitem{HuangGao1}
G.~Huang, Z.~Liu, L.~{van der Maaten}, and K.~Q. Weinberger, ``Densely
connected convolutional networks,'' in \emph{Proc. International Conference
    on Computer Vision and Pattern Recognition (CVPR)}, 2017, pp. 4700--4708.

\bibitem{Hornik4}
P.~Baldi and K.~Hornik, ``Neural networks and principal component analysis:
Learning from examples without local minima,'' \emph{Neural Networks},
vol.~2, pp. 53--58, 1989.

\bibitem{Baldi1}
P.~Baldi, ``Linear learning: Landscapes and algorithms,'' in \emph{Advances in
    Neural Information Processing Systems (NIPS 1988)}, 1988, pp. 65--72.

\bibitem{Baldi2}
P.~Baldi and Z.~Lu, ``Complex-valued autoencoders,'' \emph{Neural Networks},
vol.~33, pp. 136--147, 2012.

\bibitem{Kawaguchi1}
K.~Kawaguchi, ``Deep learning without poor local minima,'' in \emph{Proceedings
    of the 29th International Conference on Neural Information Processing Systems
    (NIPS 2016)}, 2016, pp. 586--594.

\bibitem{LeCun4}
A.~Choromanska, M.~Henaff, M.~Mathieu, G.~B. Arous, and Y.~LeCun, ``The loss
surfaces of multilayer networks,'' in \emph{Proceedings of the Eighteenth
    International Conference on Artificial Intelligence and Statistics}, 2015,
pp. 192--204.

\bibitem{LeCun5}
A.~Choromanska, Y.~LeCun, and G.~B. Arous, ``Open problem: The landscape of the
loss surfaces of multilayer networks,'' in \emph{Proceedings of The 28th
    Conference on Learning Theory}, 2015, pp. 1756--1760.

\bibitem{Kawaguchi2}
H.~Lu and K.~Kawaguchi, ``Depth creates no bad local minima,'' pp. 1--10, 2017.
[Online]. Available: {https://arxiv.org/abs/1702.08580}

\bibitem{YunCH1}
C.~Yun, S.~Sra, and A.~Jadbabaie, ``Global optimality conditions for deep
neural networks,'' in \emph{Proceedings of the 6th International Conference
    on Learning Representations (ICLR 2018)}, Vancouver, BC, Canada, 2018, pp.
1--14.

\bibitem{YBengio2}
L.~Dinh, R.~Pascanu, S.~Bengio, and Y.~Bengio, ``Sharp minima can generalize
for deep nets,'' 2017. [Online]. Available:
{https://arxiv.org/pdf/1703.04933.pdf}

\bibitem{YBengio3}
Y.~N. Dauphin, R.~Pascanu, C.~Gulcehre, S.~G. Kyunghyun~Cho, and Y.~Bengio,
``Identifying and attacking the saddle point problemin high-dimensional
non-convex optimization,'' in \emph{Advances in Neural Information Processing
    Systems (NIPS 2014)}, 2014, pp. 2933--2941.

\bibitem{ChenKe2}
K.~Chen, L.~Xu, and H.~Chi, ``Improved learning algorithms for mixture of
experts in multiclass classification,'' \emph{Neural Networks}, vol.~12,
no.~9, pp. 1229--1252, 1999.

\bibitem{LiHao1}
H.~Li, Z.~Xu, G.~Taylor, and T.~Goldstein, ``Visualizing the loss landscape of
neural nets,'' in \emph{Proceedings of the 6th International Conference on
    Learning Representations (ICLR 2018)}, Vancouver, BC, Canada, 2018, pp.
1--17.

\bibitem{Poggio7}
Q.~Liao and T.~Poggio, ``Theory of deep learning {II}: Landscape of the
empirical risk in deep learning,'' \emph{CBMM Memo No. 066}, pp. 1--45, June
2017.

\bibitem{Gunasekar1}
S.~Gunasekar, J.~D. Lee, D.~Soudry, and N.~Srebro, ``Implicit bias of gradient
descent on linear convolutional networks,'' in \emph{Proceedings of the 29th
    International Conference on Neural Information Processing Systems (NIPS
    2016)}, 2018, pp. 1--22.

\bibitem{Poggio10}
T.~Poggio, Q.~Liao, B.~Miranda, A.~Banburski, X.~Boix, and J.~Hidary, ``Theory
of deep learning {IIIb}: Generalization in deep networks,'' \emph{CBMM Memo
    No. 090}, pp. 1--37, June 2018.

\bibitem{Langford1}
J.~Langford and R.~Caruana, ``{(Not)} {Bounding} the true error,'' in
\emph{Advances in Neural Information Processing Systems 14}, T.~G.
Dietterich, S.~Becker, and Z.~Ghahramani, Eds.\hskip 1em plus 0.5em minus
0.4em\relax MIT Press, 2002, pp. 809--816. [Online]. Available:
{http://papers.nips.cc/paper/1968-not-bounding-the-true-error.pdf}

\bibitem{Dziugaite1}
G.~K. Dziugaite and D.~M.~Roy, ``Computing nonvacuous generalization bounds for
deep (stochastic) neural networks with many more parameters than training
data,'' 2017. [Online]. Available: {https://arxiv.org/pdf/1703.11008.pdf}

\bibitem{Neyshabur1}
B.~Neyshabur, ``Implicit regularization in deep learning,'' Ph.D. dissertation,
Toyota Technological Institute at Chicago, 2017.

\bibitem{Neyshabur2}
B.~Neyshabur, S.~Bhojanapalli, and N.~Srebro, ``A {PAC}-{Bayesian} approach to
spectrally-normalized margin bounds for neural networks,'' in
\emph{Proceedings of the 6th International Conference on Learning
    Representations (ICLR 2018)}, Vancouver, BC, Canada, 2018, pp. 1--9.

\bibitem{Bartlett1}
P.~L. Bartlett, D.~J. Foster, and M.~Telgarsky, ``Spectrally-normalized margin
bounds for neural networks,'' in \emph{Advances in Neural Information
    Processing Systems (NIPS 2017)}, 2017, pp. 6241--6250.

\bibitem{Neyshabur3}
B.~Neyshabur, S.~Bhojanapalli, D.~{McAllester}, and N.~Srebro, ``Exploring
generalization in deep learning,'' in \emph{Advances in Neural Information
    Processing Systems (NIPS 2017)}, 2017, pp. 5947--5956.

\bibitem{Kawaguchi3}
K.~Kawaguchi, L.~P. Kaelbling, and Y.~Bengio, ``Generalization in deep
learning,'' 2018. [Online]. Available:
{https://arxiv.org/pdf/1710.05468.pdf}

\bibitem{Brunelli11}
R.~Brunelli, ``Training neural nets through stochastic minimization,''
\emph{Neural Networks}, vol.~7, no.~9, pp. 1405--1412, 1994.

\bibitem{Patrick11}
{P. Patrick van der Smagt}, ``Minimisation methods for training feedforward
neural networks,'' \emph{Neural Networks}, vol.~7, no.~1, pp. 1--11, 1994.

\bibitem{ChenKe1}
K.~Chen, ``Deep and modular neural networks,'' \emph{Handbook on Computational
    Intelligence}, pp. 473--494, 2015, (Chapter 28).

\bibitem{Toh97}
K.-A. Toh, ``Learning from the kernel and the range space,'' in
\emph{Proceedings of the 17th IEEE/ACIS International Conference on Computer
    and Information Science}, Singapore, June 2018, pp. 417--422.

\bibitem{Toh98}
K.-A. Toh, Z.~Lin, Z.~Li, B.~Oh, and L.~Sun, ``Gradient-free learning based on
the kernel and the range space,'' \emph{https://arxiv.org/abs/1810.11581},
pp. 1--27, 27th October 2018.

\bibitem{SLCampbell1}
S.~L. Campbell and C.~D. Meyer, \emph{Generalized Inverses of Linear
    Transformations}, ({SIAM} edition of the work published by {Dover
    Publications, Inc.}, 1991)~ed.\hskip 1em plus 0.5em minus 0.4em\relax
Philadelphia, USA: Society for Industrial and Applied Mathematics, 2009.

\bibitem{Albert1}
A.~Albert, \emph{Regression and the Moore-Penrose Pseudoinverse}.\hskip 1em
plus 0.5em minus 0.4em\relax New York: Academic Press, Inc., 1972, vol.~94.

\bibitem{Adi1}
{Adi Ben-Israel} and {Thomas N.E. Greville}, \emph{Generalized Inverses: Theory
    and Applications}, 2nd~ed.\hskip 1em plus 0.5em minus 0.4em\relax New York:
Springer-Verlag, 2003.

\bibitem{MacAusland1}
R.~{MacAusland}, ``The {Moore-Penrose} inverse and least squares,'' April 2014,
lecture notes in Advanced Topics in Linear Algebra (MATH 420). [Online].
Available:
{http://buzzard.ups.edu/courses/2014spring/420projects/math420-UPS-sprin%
    g-2014-macausland-pseudo-inverse.pdf}

\bibitem{Duda1}
R.~O. Duda, P.~E. Hart, and D.~G. Stork, \emph{Pattern Classification},
2nd~ed.\hskip 1em plus 0.5em minus 0.4em\relax New York: John Wiley \& Sons,
Inc, 2001.

\bibitem{Davis1}
P.~J. Davis, \emph{Circulant Matrices}.\hskip 1em plus 0.5em minus 0.4em\relax
New York: Wiley, 1970, {ISBN} 0471057711.

\bibitem{Golub1}
G.~H. Golub and C.~F. {Van Loan}, \emph{Matrix Computations}.\hskip 1em plus
0.5em minus 0.4em\relax Johns Hopkins, 1996, {ISBN} 978-0-8018-5414-9.

\bibitem{Funa1}
K.-I. Funahashi, ``On the approximate realization of continuous mappings by
neural networks,'' \emph{Neural Networks}, vol.~2, no.~3, pp. 183--192,
1989.

\bibitem{Hornik1}
K.~Hornik, M.~Stinchcombe, and H.~White, ``Multi-layer feedforward networks are
universal approximators,'' \emph{Neural Networks}, vol.~2, no.~5, pp.
359--366, 1989.

\bibitem{Cybenko1}
G.~Cybenko, ``Approximations by superpositions of a sigmoidal function,''
\emph{Math. Cont. Signal \& Systems}, vol.~2, pp. 303--314, 1989.

\bibitem{HechNiel3}
R.~Hecht-Nielsen, ``{Kolmogorov}'s mapping neural network existence theorem,''
in \emph{Proceedings of IEEE First International Conference on Neural
    Networks (ICNN)}, vol. III, 1987, pp. 11--14.

\bibitem{Kurkova2}
V.~K{\.u}rkov{\'a}, ``{Kolmogorov}'s theorem and multilayer neural networks,''
\emph{Neural Networks}, vol.~5, no.~3, pp. 501--506, 1992.

\bibitem{Leshno1}
M.~Leshno, V.~Y. Lin, A.~Pinkus, and S.~Schocken, ``Multilayer feedforward
networks with a nonpolynomial activation function can approximate any
function,'' \emph{Neural Networks}, vol.~6, pp. 861--867, 1993.

\bibitem{ZhangCY1}
C.~Zhang, S.~Bengio, M.~Hardt, B.~Recht, and O.~Vinyals, ``Understanding deep
learning requires rethinking generalization,'' in \emph{Proceedings of the
    5th International Conference on Learning Representations (ICLR 2017)},
Toulon, France, 2017, pp. 1--15.

\bibitem{Hastie01}
T.~Hastie, R.~Tibshirani, and J.~Friedman, \emph{The Elements of Statistical
    Learning: Data Mining, Inference, and Prediction}.\hskip 1em plus 0.5em minus
0.4em\relax Canada: Springer, 2001.

\bibitem{ChenKe3}
K.~Chen and D.~L. Wang, ``Perceiving geometric patterns: from spirals to
inside/outside relations,'' \emph{IEEE Transactions on Neural Networks},
vol.~12, no.~5, pp. 1084--1102, 2001.

\bibitem{UCI1b}

M.~Lichman, ``{UCI} machine learning repository,'' 2013. [Online]. Available:
{http://archive.ics.uci.edu/ml}

\bibitem{Toh39}
K.-A. Toh, Q.-L. Tran, and D.~Srinivasan, ``Benchmarking a reduced multivariate
polynomial pattern classifier,'' \emph{IEEE Trans. on Pattern Analysis and
    Machine Intelligence}, vol.~26, no.~6, pp. 740--755, 2004.

\bibitem{Matlab}
{The MathWorks}, ``Matlab and simulink,'' in
\emph{[http://www.mathworks.com/]}, 2017.


\end{thebibliography}

\end{document}